\documentclass[a4paper, 10pt]{article}

\usepackage[a4paper,left=3cm,right=3cm,top=2.5cm,bottom=2.5cm]{geometry}
\usepackage{hyperref}
\usepackage{subfig}
\usepackage{pgfplots}
\usepackage{pgfplotstable}
\usepackage{mathtools}
\usepackage{multicol}
\usepackage{comment}
\usepackage{booktabs}
\pgfplotsset{compat=1.5}
\usepackage{amssymb}
\usepackage{url}
\usepackage{bm}

\usepackage{pdfsync}
\usepackage{float}
\usepackage{tabularx}
\usepackage{enumerate}
\usepackage{array}
\usepackage{xspace}
\usepackage{tikz}
\usepackage{tikz-cd}
\usepackage{tikzsymbols}
\usetikzlibrary{calc,trees,positioning,arrows,chains,shapes.geometric,%
    decorations.pathreplacing,decorations.pathmorphing,shapes,%
    matrix,shapes.symbols, decorations.markings, patterns,fit}

\usepackage{siunitx}

\usepackage{authblk}

\usepackage[draft,inline,marginclue]{fixme}
\usepackage{mathrsfs}
\usepackage{color, colortbl}
\usepackage{multirow}

\usepackage{amsthm}
\usepackage{amsmath}
\usepackage{stmaryrd}
\usepackage{algorithm,algpseudocode}
\usepackage{graphicx}
\usepackage{booktabs}

\usepackage[noabbrev]{cleveref}

\FXRegisterAuthor{mt}{amt}{MT} \FXRegisterAuthor{fr}{afr}{FR}

\newtheorem{theorem}{Theorem}

\newtheorem{corollary}{Corollary}
\newtheorem{definition}{Definition}
\newtheorem{problem}{Problem}
\theoremstyle{definition}
\newtheorem{rmk}{Remark}
\newtheorem{ass}{Assumption}

\newcommand{\mupar}{\ensuremath{\boldsymbol{\mu}}}

\newcommand{\XX}{\ensuremath{\mathbf{X}}}
\newcommand{\YY}{\ensuremath{\mathbf{Y}}}
\newcommand{\ZZ}{\ensuremath{\mathbf{Z}}}
\newcommand{\RR}{\ensuremath{\mathbb{R}}}
\newcommand{\xx}{\ensuremath{\mathbf{x}}}
\newcommand{\yy}{\ensuremath{\mathbf{y}}}

\DeclareMathOperator*{\argmin}{\arg\!\min}

\newcolumntype{C}[1]{>{\centering\arraybackslash}m{#1}}

\definecolor{Gray}{gray}{0.9}

\crefname{theorem}{theorem}{theorems}
\crefname{problem}{problem}{problems}
\crefname{rmk}{remark}{remarks}
\crefname{lemma}{lemma}{lemmas}
\crefname{ass}{assumption}{assumptions}

\begin{document}

\title{A local approach to parameter space reduction for regression and
  classification tasks}

\author[]{Francesco~Romor\footnote{francesco.romor@sissa.it}}
\author[]{Marco~Tezzele\footnote{marco.tezzele@sissa.it}}
\author[]{Gianluigi~Rozza\footnote{gianluigi.rozza@sissa.it}}

\affil{Mathematics Area, mathLab, SISSA, via Bonomea 265, I-34136 Trieste,
  Italy}

\maketitle

\begin{abstract}
Parameter space reduction has been proved to be a crucial
  tool to speed-up the execution of many numerical tasks such as
  optimization, inverse problems, sensitivity analysis, and surrogate
  models' design, especially when in presence of high-dimensional
  parametrized systems. 
In this work we propose a new method called local active subspaces
(LAS), which explores the synergies of active subspaces with
supervised clustering techniques in order to carry out a more efficient
dimension reduction in the parameter space. The clustering is
performed without losing the input-output relations by introducing a
distance metric induced by the global active subspace. 
We present two possible clustering algorithms: K-medoids and a
hierarchical top-down approach, which is able to impose a variety of
subdivision criteria specifically tailored for parameter space reduction
tasks.  
This method is particularly useful for the community working on
surrogate modelling. Frequently, the parameter space presents subdomains
where the objective function of interest varies less on average along different directions. So, it could be approximated more accurately if restricted to
those subdomains and studied separately. 
We tested the new method over several numerical experiments of
increasing complexity, we show how to deal with vectorial outputs, and
how to classify the different regions with respect to the local active subspace dimension. 
Employing this classification technique as a preprocessing step in the parameter space,
or output space in case of vectorial outputs, brings remarkable
results for the purpose of surrogate modelling.
\end{abstract}

\tableofcontents
\listoffixmes

\section{Introduction}
\label{sec:intro}
Parameter space reduction~\cite{constantine2015active,
  tezzele2022aroma16} is a rapidly growing field of interest which
plays a key role in fighting the curse of dimensionality. The need of reducing the
number of design inputs is particularly important in engineering for advanced
CFD simulations to model complex phenomena, especially in the broader context of
model order reduction~\cite{chinestaenc2017, chinesta2011,
  beohparour17, morhandbook2020} and industrial numerical
pipelines~\cite{brunton2021data, brunton2019data, rozza2018advances}.

  Active subspaces~\cite{constantine2015active} is one of the most used techniques
  for linear reduction in input spaces. It has been proved useful in many
  numerical tasks such as regression, using a multi-fidelity data fusion approach
  with a surrogate model built on top of the AS as low-fidelity
  model~\cite{romor2023multi}, shape
  optimization~\cite{lukaczyk2014active, tezzele2018dimension, boncoraglio2020model} and a
  coupling with the genetic algorithm to enhance its
  performance~\cite{demo2020asga, demo2021hull}, inverse
  problems~\cite{zahm2022certified}, and uncertainty
  quantification~\cite{cortesi2020forward}. It has also been used to
  enhance classical model order reduction techniques such as
  POD-Galerkin~\cite{tezzele2018combined}, and POD with
  interpolation~\cite{demo2019cras, tezzele2023multi}. Other attempts towards nonlinear parameter
  space reduction have been proposed recently: kernel-based active
  subspaces~\cite{romor2022kas}, nonlinear level-set
  learning~\cite{zhang2019learning}, and active manifold~\cite{bridges2019active}
  are the most promising. In~\cite{chen2019hessian}, instead, they
  project the input parameters onto a low-dimensional subspace spanned by the
  eigenvectors of the Hessian corresponding to its dominating
  eigenvalues.

In this work we propose a new local approach for parameter space dimensionality
reduction for both regression and classification tasks, called Local Active
Subspaces (LAS).
In our work we do not simply apply a clustering technique to
preprocess the input data, we propose a supervised metric induced by the
presence of a global active subspace. The directions individuated by local
active subspaces are locally linear, and they better capture the latent manifold
of the target function. 

From a wider point of view, there is an analogy between local parameter space
reduction and local model order reduction. With the latter, we mean both a
spatial domain decomposition approach for model order reduction of parametric
PDEs in a spatial domain $\Omega\subset\mathbb{R}^{d}$ and a local reduction
approach in the parameter space. As representatives methods for the first paradigm
we report the reduced basis element method~\cite{lovgren2006reduced}, which
combines the reduced basis method in each subdomains with a mortar type method
at the interfaces, and more in general domain decomposition methods applied to model order reduction. For the second approach we cite the interpolation method in
the Grassmannian manifold of the reduced
subspaces~\cite{amsallem2008interpolation}; in particular in
\cite{daniel2020model} the K-medoids clustering algorithm with Grassmann metric
is applied to the discrete Grassmann manifold of the training snapshots as a
step to perform local model order reduction. With this work we 
fill the gap in the literature regarding localization methods in the
context of parameter space reduction.

Other methods have been developed in the last years exploiting
the localization idea. We mention localized slice inverse regression
(LSIR)~\cite{wu2009localized} which uses local information of the slices for
supervised regression and semi-supervised classification. LSIR improves local
discriminant information~\cite{hastie1996discriminant} and local Fischer
discriminant analysis~\cite{sugiyama2007dimensionality} with more efficient
computations for classification problems. The main difference between slice
inverse regression (SIR)~\cite{li1991sliced} and AS is in the construction of
the projection matrix. While SIR needs the elliptic assumption, AS exploits the
gradients of the function of interest with respect to the input parameters.
Recently, a new work on the subject was
disclosed~\cite{xiong2020clustered}. Here we emphasize the differences
and the original contributions of our work: 
1) we implemented
hierarchical top-down clustering applying K-medoids 
with a new metric that includes the gradient information through the active
subspace. In~\cite{xiong2020clustered} they employed hierarchical bottom-up
clustering with unweighted average linkage and a distance obtained as a
weighted sum of the Euclidean distance of the inputs and the cosine of the
angle between the corresponding gradients;
2) we included for vector-valued objective
functions and answered questions about the employment of the
new method to decrease the ridge approximation error with respect to a
global approach;
3) we also focused on classification
algorithms and we devised a method to classify the inputs based on the local
active subspace dimension with different techniques, including the use of
the Grassmannian metric;
4) our benchmarks include vector-valued objective functions from
computational fluid dynamics. We also show that clustering the outputs with
our classification algorithms as pre-processing step leads to more efficient
surrogate models.

This work is organized as follows: in \cref{sec:as} we briefly review the active
subspaces method, in \cref{sec:local} we introduce the clustering algorithms
used and the supervised distance metric based on the presence of a global active
subspace, focusing on the construction of response surfaces and providing
theoretical considerations. In \cref{sec:classification} we present the algorithms to
exploit LAS for classification. We provide extensive numerical results in
\cref{sec:results} from simple illustrative bidimensional dataset to
high-dimensional scalar and vector-valued functions. Finally, in
\cref{sec:the_end} we draw conclusions and future perspectives.

\section{Active subspaces for parameter space reduction}
\label{sec:as}
Active subspaces (AS)~\cite{constantine2015active} are usually employed as dimension reduction method to
unveil a lower dimensional structure of a function of interest $f$, or provide a global sensitivity measure not necessarily aligned with the coordinate axes~\cite{sullivan2015introduction}. Through
spectral considerations about the second moment matrix of $f$, the AS
identify a set of linear combinations of the input parameters along which $f$
varies the most on average.

We make some general assumptions about the inputs and function of interest~\cite{constantine2015active,zahm2020gradient, sullivan2015introduction}. Let us introduce the inputs as an absolutely continuous random vector $\XX$ with values in $\RR^n$ and probability distribution $\boldsymbol{\mu}$. We represent with $\mathcal{X} \subset \RR^n$ the support of $\boldsymbol{\mu}$ and as such our parameter space. We want to compute the active subspace of a real-valued
function $f:(\mathcal{X}, \mathcal{B}(\RR^n), \boldsymbol{\mu})\rightarrow\RR$, where $\mathcal{B}(\RR^n)$ is the Borel $\sigma$-algebra of $\mathbb{R}^n$. We denote with $\mathbf{x}\in\mathcal{X}$ an element in the space of parameters and with $\{\mathbf{x}_i\}_i$ a set of realizations of $\XX$.

An extension to vector-valued functions has been
presented in~\cite{zahm2020gradient} and extended for kernel-based AS
in~\cite{romor2022kas}. Even if in this section we focus only on scalar
functions, the following considerations can be carried over to the multivariate
case without too much effort.

Let $\Sigma$ be the second moment matrix of $\nabla f$ defined as
\begin{equation}
\label{eq:cov_matrix}
\Sigma := \mathbb{E}_{\boldsymbol{\mu}}\, [\nabla_{\xx} f \, \nabla_{\xx} f
^T] =\int (\nabla_{\xx} f) ( \nabla_{\xx} f )^T\, d \mupar,
\end{equation}
where $\mathbb{E}_{\boldsymbol{\mu}}$ denotes the expected value with respect to $\boldsymbol{\mu}$, and $\nabla_{\xx} f = \nabla f(\xx) =
\left [ \frac{\partial f}{\partial \xx_1}, \dots, \frac{\partial f}{\partial
\xx_n} \right ]^T$ is the column vector of partial derivatives of $f$. Its
real eigenvalue decomposition reads $\Sigma = \mathbf{W} \Lambda \mathbf{W}^T$.
We can retain the most energetic eigenpairs by looking at the spectral
decay of the matrix $\Sigma$. The number $r$ of eigenpairs we select is the
active subspace dimension, and the span of the corresponding eigenvectors
defines the active subspace. The partition is the following
\begin{equation}
\mathbf{\Lambda} =   \begin{bmatrix} \mathbf{\Lambda}_1 & \\
                                     &
                                     \mathbf{\Lambda}_2\end{bmatrix},
\qquad
\mathbf{W} = \left [ \mathbf{W}_1 \quad \mathbf{W}_2 \right ],
\end{equation}
where $\mathbf{\Lambda}_1 = \text{diag}(\lambda_1, \dots, \lambda_r)$, and
$\mathbf{W}_1$ contains the first $r$ eigenvectors arranged by columns. With
this matrix we can project the input parameters onto the active subspace, and
its orthogonal complement, that is the inactive subspace, as follows:
\begin{equation}
  \label{eq:active_and_inactive_def}
\YY = P_{r}(\XX)=\mathbf{W}_1\mathbf{W}_1^T \XX \in \mathbb{R}^n, \qquad
\ZZ = (I-P_{r})(\XX)=\mathbf{W}_2\mathbf{W}_2^T \XX \in \mathbb{R}^n ,
\end{equation}
with $P_{r}:\RR^n\rightarrow\RR^n$ the linear projection operator $P_{r} := \mathbf{W}_1\mathbf{W}_1^T$.
The selection of the active subspace dimension $r$ can be set a priori, or by
looking at the presence of a spectral gap~\cite{constantine2015active}, or by
imposing a cumulative energy threshold for the eigenvalues.

We will consider the problem of ridge approximation~\cite{pinkus2015ridge} in our applications. The AS are, in fact, the minimizers of an upper bound of the ridge approximation error.

\begin{definition}[Ridge approximation]
  \label{def:ridge_pb}
  Given
  $r\in\mathbb{N},\,r\ll n$ and a tolerance $\epsilon\geq 0$, find the
  profile function $h:(\mathbb{R}^n, \mathcal{B}(\mathbb{R}^n),
  \boldsymbol{\mu})\rightarrow \RR$ and the $r$-rank projection
  $P_r:\mathbb{R}^n\rightarrow\mathbb{R}^n$ such that the following upper bound on the ridge approximation error is satisfied
  \begin{equation}
  \label{eq:ridge_ineq}
  \mathbb{E}_{\boldsymbol{\mu}}[\lVert
  f-h\circ P_r\rVert_2^2]\leq \epsilon^2,
  \end{equation}
  where $\lVert\cdot\rVert_2$ is the $L^2$-norm of $\RR$.
\end{definition}

For a fixed projection $P_{r}$ the optimal profile $\Tilde{h}$ is given by the conditioned random variable $\mathbb{E}_{\boldsymbol{\mu}}[f\vert P_{r}]$. Under the additional assumptions on the probability distribution $\boldsymbol{\mu}$, reported in~\cref{app:assumptions} of the Appendix, the AS can indeed be defined as a minimizer of an upper bound of the ridge approximation error~\cite{constantine2015active,zahm2020gradient, parente2020generalized, romor2022kas}. The proof is a direct consequence of the Poincar\'e inequality and standard properties of eigenspaces, and for this specific version of the theorem it can be found in~\cite{parente2020generalized}.
  
\begin{theorem}[Definition of AS through ridge approximation]
  \label{theo:existence}
  The solution $P_{r}$ of the ridge
  approximation problem in~\cref{def:ridge_pb}, with optimal profile $\Tilde{h}=\mathbb{E}_{\boldsymbol{\mu}}[f\vert P_{r}]$, is the orthogonal
  projector to the eigenspace of the first $r$-eigenvalues of $\Sigma $ ordered by magnitude
  \begin{equation*}
  \Sigma v_i=\lambda_i v_i\qquad\forall i\in\{1,\dots,n\},\qquad
  \Tilde{P}_{r}=\sum_{j=1}^{r}v_{j}\otimes v_{j},
  \end{equation*}
  with $r\in\mathbb{N}$ chosen such that
  \begin{align*}
  \mathbb{E}_{\boldsymbol{\mu}}\left[\lVert
  f-\Tilde{h}\rVert^2_{2}\right]&\leq C(C_{P}, \tau) \argmin_{\substack{P^{2}=P, P=P^{T},\\
  \text{rank}(P)=r}}\,\mathbb{E}_{\boldsymbol{\mu}}[\| (Id-P)\nabla
  f\|^{2}]^{\frac{1}{1+\tau}}\\
  &\leq C(C_{P}, \tau)\,\left(\sum_{i=r+1}^{m}\lambda_{i}\right)^{\frac{1}{1+\tau}}\leq \epsilon^2.
  \end{align*}
  with $C(C_{P}, \tau)$ a constant depending on $\tau>0$ related to the choice of
  $\boldsymbol{\mu}$ and on the Poincar\'e constant $C_P$, and $\Tilde{h}=\mathbb{E}_{\boldsymbol{\mu}}[f\vert\sigma(P_r) ]$ is the conditional expectation of $f$ given the $\sigma$-algebra
  generated by the random variable $P_r\circ\XX$.
  \end{theorem}

  To ease the notation, in the following we will
  consider only the first three classes of probability distribution in
  the assumptions of~\cref{app:assumptions} in the Appendix, such
  that $\tau=0$.

\section{Localized parameter space reduction}
\label{sec:local}
Sometimes we do not have \textit{a priori} knowledge about the target function's
behaviour in a particular parameter space region. This could lead to a
poor selection of the parameters range, hugely affecting optimization
tasks. In these cases, a preprocessing of the data using a
clustering technique could be highly beneficial. With a clustering of the input
parameters, we can treat each subregion separately, and thus capture more accurately the target function's variability. This is always true for any function
of interest, but for functions with global lower intrinsic dimensionality we can
exploit such structure to enhance the clustering. To this end, we propose a new
distance metric for K-medoids and  hierarchical top-down clustering methods
which exploits the global active subspace of the target function. By applying AS
on each cluster we find the optimal rotation of the corresponding subregion of
the input domain, which aligns the data along the active subspace of a given
dimension.

In this section, we make some theoretical considerations regarding ridge approximation applied to partitions of the parameter space and review three clustering
methods~\cite{han2011data}: K-means, K-medoids, and hierarchical
top-down clustering~\cite{kaufman2009finding, murphy2012machine}. We
are going to use K-means as the baseline since the input parameter space is
assumed to be a hyperrectangle. This assumption covers the
  majority of the practical test cases in the reduced order modeling community.

\subsection{Ridge approximation with clustering and active subspaces}

Regardless of the choice of clustering algorithm, given a partition of the parameter space we want to perform ridge approximation with AS in each subdomain. We will introduce some definitions and make some remarks to clarify the setting.
The function of interest $f$ represents scalar outputs, but the following statements can be extended to
vector-valued outputs as well.

\begin{definition}[Local ridge approximation with active subspaces]
  \label{def:local ridge approximation}
  Given a partition of the domain $\mathcal{P}:=\{S_{i}\}_{i\in\{1,\dots, d\}}$
      and a map $r:\mathcal{P}\rightarrow\{1,\dots,n_{r}\},\,n_{r}\ll n$
      representing the local reduced dimension, the local ridge approximation
      with active subspaces of ($f$, $\boldsymbol{\mu}$) is the function
      $R_{AS}(r,f,
      \boldsymbol{\mu}):\mathcal{X}\subset\mathbb{R}^{n}\rightarrow\mathbb{R}$
      that is defined locally for every $S_i\in\mathcal{P}$ as
  \begin{equation}
      g \vert_{S_i} = \mathbb{E}_{\boldsymbol{\mu}_{i}}[f\vert
      P_{r(S_i), i}] ,
  \end{equation}
  where
  $\boldsymbol{\mu}_{i}:=(1/\boldsymbol{\mu}({S_{i}}))\cdot
  \boldsymbol{\mu} \vert_{S_{i}} \in \mathbb{R}^n$,
  and $P_{r(S_i), i} : S_i \subset \RR^n \rightarrow \RR^n$ is the orthogonal projector
  with rank $r$ that satisfies the minimization problem:
  \begin{equation}
      P_{r(S_i), i} = \argmin_{\substack{P^{2}=P, P=P^{T},\\
      \text{rank}(P)=r}}\,\mathbb{E}_{\boldsymbol{\mu}_{i}}\| (Id-P)\nabla
    f\|^{2}.
  \end{equation}
\end{definition}

With this definition we can state the
problem of local ridge approximation with active subspaces.
\begin{problem}[Minimizers $(\mathcal{P}, r)$ of the ridge approximation error]
\label{pb:partition_minimizer}
Find the partition $\mathcal{P}$ of the domain
$\mathcal{X}\subset\mathbb{R}^{n}$ and the local reduced dimension map
$r:\mathcal{P}\rightarrow\{1,\dots,n_{r}\},\,n_{r} \ll n$, such that the
$L^2$-error between the objective function $f$ and its local ridge approximation
with active subspaces is minimized.
\begin{equation}
  \mathbb{E}_{\boldsymbol{\mu}}\left[\|f-R_{\text{AS}}(r, f)\|^{2}\right]  =
    \sum_{S_{i}\in\mathcal{P}} \mathbb{E}_{\boldsymbol{\mu}} \left[ \|
      f \vert_{S_{i}}- \mathbb{E}_{\boldsymbol{\mu}_{i}} [f  \vert  P_{r(S_{i}),
    i}]\|^{2}\right].
\end{equation}
\end{problem}

Assuming that the subspace Poincar\'e
inequality~\cite{parente2020generalized} is valid also for ($f,
\boldsymbol{\mu}$) restricted to the elements of the partition $\mathcal{P}$, a
straightforward bound is obtained by applying the Poincaré inequality for every
element of the partition
\begin{align*}
    \mathbb{E}_{\boldsymbol{\mu}}\left[\|f-R_{\text{AS}}(r, f)\|^{2}\right] & =
    \sum_{S_{i}\in\mathcal{P}} \mathbb{E}_{\boldsymbol{\mu}} \left[ \|
      f \vert_{S_{i}}- \mathbb{E}_{\boldsymbol{\mu}_{i}} [f  \vert  P_{r(S_{i}),
    S_{i}}]\|^{2}\right]                      \\
    & \lesssim  \sum_{S_{i}\in\mathcal{P}}\mathbb{E}_{\boldsymbol{\mu}}\left[ \|(Id-P_{r(S_{i}), i})\nabla f\|^{2}\right].
\end{align*}

To obtain the previous upper bound, we made an assumption about the Poincaré
subspace inequality that in general is not satisfied by any probability measure
$\mupar$ chosen: the assumptions on the probability distributions
$\{\boldsymbol{\mu}_i\}_{i=1}^d$ in~\cref{app:assumptions} of the
Appendix have to be satisfied at each subdomain $\{S_i\}_{i=1}^d$. 

For the moment we will consider the local reduced dimension map $r$ constant and,
in general, the codomain of $r$ is a subset of $\{1,\dots,n_{r}\},\,n_{r} \ll n$.

The previous bound suggests that a good indicator for refinement could be
represented by the sum of the residual eigenvalues $\{\lambda_{S_i, j}\}_{j=r_{S_i}}^{m}$ of the local correlation matrices, for every $S_i\in\mathcal{P}$:
\begin{equation*}
  \mathbb{E}_{\boldsymbol{\mu}}\left[\|f-R_{\text{AS}}(r, f)\|^{2}\right]  \lesssim
  \sum_{S_{i}\in\mathcal{P}} \sum_{j=r(S_i)+1}^{m}\lambda_{S_i, j}.
\end{equation*}

We also
have the following immediate result that hints to indefinitely many successive
refinements to lower the $L^2$-error ridge approximation error.

\begin{rmk}[Relationships between the upper bounds of consecutive refinements]
    \label{rmk:upper bounds}
    Considering the sum over the number of refined clusters $cl\in\{1,\dots,d\}$
    we have that
    \begin{align}
        \int_{\mathcal{X}}\| (Id-P_{r})\nabla f\|^{2}\,
        d\boldsymbol{\mu} &=\sum_{cl=1}^{d}\int_{S_{cl}\subset\mathcal{X}}\|
         (Id-P_{r})\nabla f\|^{2}\, d\boldsymbol{\mu} \nonumber \\
      &\geq  \sum_{cl=1}^{d}\int_{S_{cl}\subset\mathcal{X}}\| (Id-P_{r,
        cl})\nabla f\|^{2}\, d\boldsymbol{\mu},
        \label{eq:indicator}
    \end{align}
    since the projectors $\{P_{r, cl}\}_{cl\in\{1,\dots,d\}}$ are the minimizers
    of
    \begin{equation}
        P_{r, cl} = \argmin_{\substack{P^{2}=P, P=P^{T},\\
        \text{rank}(P)=r}}\quad\int_{S_{cl}\subset\mathcal{X}}\|
        (Id-P)\nabla f\|^{2}\, d\boldsymbol{\mu} .
    \end{equation}
    The RHS of \cref{eq:indicator} can be used as indicator for refinement. We
    remark that since the refinements increase the decay of the
    eigenvalues in the RHS of \cref{eq:indicator}, the choice of the dimension
    of the active subspace may be shifted towards lower values to achieve further dimension reduction for the same accuracy, as we are going
    to show in the numerical experiments, in~\cref{sec:results}.
\end{rmk}

Unfortunately, the minimizers of the ridge approximation error and of the upper
bound are not generally the same:
\begin{equation*}
  \argmin_{\{P_{r(S_i), i}\}_{S_i\in\mathcal{P}}}
  \mathbb{E}_{\boldsymbol{\mu}} \left[\|f-R_{\text{AS}}(r, f)\|^{2}\right] \neq \argmin_{\{P_{r(S_i),
      S_i}\}_{S_i\in\mathcal{P}}} \sum_{S_{i}\in\mathcal{P}} \mathbb{E}_{\boldsymbol{\mu}_i}\left[ \|(Id-P_{r(S_{i}), i})\nabla f\|^{2}\right].
\end{equation*}
There is a counterexample for the non localized case
in~\cite{zahm2020gradient}. We start from this counterexample to show that in
general the $L^2$-error of the local ridge approximation does not decrease
between consequent refinements, even if the indicator from the RHS of
\cref{eq:indicator} does, as stated in the previous remark.

\begin{corollary}[Counterexample for indefinite refinement as optimal clustering criterion]
    \label{cor:counterexample refinement}
    Let $\mathcal{P} = \{A, B, C\}$ be a partition of $\mathcal{X}=[-1,1]^{2}$
    such that $A=[-1, \epsilon]\times[-1, 1]$, $B = [-\epsilon,
    \epsilon]\times[-1, 1]$, and $C=[\epsilon, 1]\times[-1, 1]$. Let
    $\boldsymbol{\mu}$ be the uniform probability distribution on $\mathcal{X}$.
    The objective function we want to approximate is
    \begin{equation}
        f:\mathcal{X}\subset\mathbb{R}^{2}\rightarrow\mathbb{R}, \quad
        f = \begin{cases}
            x_{1} + \epsilon,                                        & \mathbf{x}\in A, \\
            x_{1}(x_{1}+\epsilon)(x_{1}-\epsilon)\cos(\omega x_{2}), & \mathbf{x}\in B, \\
            x_{1} - \epsilon,                                        & \mathbf{x}\in C,
        \end{cases}
    \end{equation}
    with local reduced dimension map $r(A)=r(B)=r(C)=1$. There exist
    $\epsilon>0, \omega>0$, such that
    \begin{equation}
        \label{eq:counterexample}
        \mathbb{E}_{\boldsymbol{\mu}}\left[\|f-R_{AS}(r, f, \mupar)\|^{2}\right]\nonumber
        \geq
        \mathbb{E}_{\boldsymbol{\mu}}\left[\|f-\mathbb{E}_{\boldsymbol{\mu}}\left[f
            \vert P_{1, \mathcal{X}}\right]\|^{2}\right],
    \end{equation}
    where $P_{1, \mathcal{X}}$ is the optimal projector on the whole domain
    $\mathcal{X}$ with one-dimensional active subspace.
\end{corollary}
\begin{proof}
    The proof is reported in \cref{sec:proof_counter} of the Appendix.
\end{proof}
The heuristics behind the previous proof rests on the fact that ridge
approximation with active subspaces performs poorly when the objective function
has a high variation. The counterexample is valid whenever the global projector
$P_{1, \mathcal{X}}$ is the minimizer of a local $L^2$ ridge approximation error
for which the minimizer of the gradient-based indicator in \cref{eq:indicator}
does not coincide. This leaves us with an indicator in \cref{eq:indicator} that
does not guarantee a non increasing $L^2$-error decay for subsequent
refinements, but is nonetheless useful in practice.

We conclude the section with some remarks about the response surface design
through the ridge approximation with active subspaces.
\begin{rmk}[Approximation of the optimal profile]
    \label{rmk:approx_opt_profile}
    In practice we do not consider the optimal profile $h(\yy) =
        \mathbb{E}_{\boldsymbol{\mu}}\left[f \vert \sigma(P_{r})\right](\yy)$ but we
        employ the approximation $h(\yy)=f(\yy)=f(P_{r}\mathbf{x})$. The reason lies on
        the fact that to approximate the optimal profile at the values
        $\{\yy_{i}\}_{i}$, additional samples from the conditional
        distribution~$p(z \vert \yy_{i}=P_{r}\mathbf{x})$ must be obtained; even if the
        accuracy of the ridge approximation could benefit from it, this is not
        always possible in practice because of the difficulty to sample from the
        conditional distribution or because of computational budget constraints.
\end{rmk}

If the data is split in training, validation, and test set, the local $R^2$
score on the validation set can be used as indicator for refinement.

\begin{rmk}[Estimator based on local $R^2$ scores]
    \label{rmk:local r2}
    The $R^2$ score of a single cluster can be written with respect to the $R^2$
    scores $\{R^2_l\}_{l\in\{1,\dots,d\}}$ relative to the clusters of the
    subsequent refinement. Let the sum be over the refinement clusters
    $l\in\{1,\dots,d\}$, we have
    \begin{align}
        \label{eq:r2_def}
        R^2 & = 1 -\frac{\mathbb{E}[\| f - \mathbb{E}[f \vert P_{r}]\rVert^{2}]}{\text{Var}(f)} = 1 -
        \sum_{l=1}^{d}\frac{\mathbb{E}[\| f \vert_{S_l}
        - \mathbb{E}[f \vert P_{r, l}]\rVert^{2}]}{\text{Var}(f)}                         \\
            & =1 -
        \sum_{l=1}^{d}\frac{\text{Var}(f \vert_{S_l})}{\text{Var}(f)}\cdot\frac{\mathbb{E}[\|
        f\vert_{S_l} - \mathbb{E}[f \vert P_{r, l}]\rVert^{2}]}{\text{Var}(f\vert_{S_l})} = 1 -
        \sum_{l=1}^{d}\frac{\text{Var}(f\vert_{S_l})}{\text{Var}(f)}\cdot
        (1 - R^2_l) \nonumber,
    \end{align}
    which, substituting with the empirical variance, becomes
    \begin{equation}
        R^2_{\text{emp}} = 1 - \sum_{l=1}^{d}
        \frac{\text{Var}_{\text{emp}}(f\vert_{S_l})}{\text{Var}_{\text{emp}}(f)}\cdot
        (1 - R^2_{\text{emp}; l})\cdot \frac{N_l -1}{N-1},
    \end{equation}
    where $R^2_{\text{emp}; l}$ is the empirical local $R^2$ score relative to
    cluster number $l$. The definition can be extended for component-wise
    vector-valued objective functions $f$. The numerical results shown in
    \cref{sec:results} consider the mean $R^2$ score along the components when
    the output is vectorial.
\end{rmk}

In practice every expectation is approximated with simple Monte Carlo, and without the
number of training samples increasing, the confidence on the approximation is
lower and lower, the more the domain is refined. This is taken into
consideration while clustering, fixing a minimum number of samples per cluster for example.

The \cref{app:upperbound approx AS} in the Appendix clarifies the link between the number of Monte Carlo
samples, the numerical method chosen for the discretization of the integral
$\mathbb{E}_{\boldsymbol{\mu}}\left[\nabla f\otimes\nabla f\right]$, and the
approximation of the active subspace. For example for deterministic models, one
could employ the more efficient Sobol' sequence or a Latin hypercube sampling; if
$f$ is more regular and the parameter space dimension is not too high one could
employ tensor product Gauss quadrature rule. See for example
\cite{sullivan2015introduction}.

Before introducing the clustering algorithms we will employ, we specify that the
partition $\mathcal{P}=\{S_{i}\}_{i\in\{1,\dots,d\}}$ is defined by the
decision boundaries of the clustering algorithm chosen.

\subsection{K-means clustering}
\label{sec:kmeans}
We recall the K-means clustering algorithm. Let $\{ x_i \}_{i=1}^N$ be a set of $N$ samples in $\RR^{N_F}$, where $N_F$
denotes the number of features. The K-means algorithm divides this set into $K$
disjoint clusters $S = \{ S_j \}_{j=1}^K$, with $S_l \cap S_m = \emptyset$ for
$1 \leq l, m \leq K$ and $l \neq m$. The partitioning quality is assessed by a
function which aims for high intracluster similarity and low intercluster
similarity. For K-means this is done by minimizing the total within-cluster
sum-of-squares criterion $W_T$, which reads as
\begin{equation}
  \label{eq:means_criterion}
W_T (S) := \sum_{j=1}^K W(S_j) = \sum_{j=1}^K \sum_{x_i \in S_j} \| x_i - c_j \|_{L^2}^2 ,
\end{equation}
where $c_j$ is the centroid describing the cluster $S_j$. A centroid of a
cluster is defined as the mean of all the points included in that cluster. This
means that the centroids are, in general, different from the samples $x_i$.

K-means is sensitive to outliers, since they can distort the mean value of a
cluster and thus affecting the assignment of the rest of the data.



\subsection{K-medoids clustering with active subspaces-based metric}
\label{sec:kmedoids}
In order to overcome some limitations of the K-means algorithm, such as
sensitivity to outliers, we can use K-medoids clustering technique~\cite{kaufman2009finding,
park2009simple, schubert2019faster, maranzana1963location}. It uses an actual
sample as cluster representative (i.e. medoid) instead of the mean of the
samples within the cluster.

Following the notation introduced in the previous section, let $m_j$ be the
medoid describing the cluster $S_j$. The partitioning method is performed by
minimizing the sum of the dissimilarities between the samples within a cluster
and the corresponding medoid. To this end, an absolute-error criterion $E$ is
used, which reads as
\begin{equation}
\label{eq:medoids_criterion}
E (S) := \sum_{j=1}^K E(S_j) = \sum_{j=1}^K \sum_{x_i \in S_j} \| x_i - m_j \|.
\end{equation}
By looking at the formula above it is clear that the use of a data point to
represent each cluster's center allows the use of any distance metric for
clustering. We remark that the choice of the Euclidean distance does not produce
the same results as K-means because of the different references representing the
clusters.

We propose a new supervised distance metric inspired by the global active
subspace of the function $f$ we want to approximate. We define a scaled $L^2$
norm using the eigenpairs of the second moment matrix of $\nabla f$, which is
the matrix from which we calculate the global active subspace:
\begin{equation}
  \label{eq:as_norm}
  \| x_i - x_j \|_{\Lambda} = \sqrt{(x_i - x_j)^T \mathbf{W} \Lambda
    \mathbf{W}^T (x_i - x_j)} ,
\end{equation}
where $\Lambda$ stands for the diagonal matrix with entries the eigenvalues of
\cref{eq:cov_matrix}, and $\mathbf{W}$ is the eigenvectors matrix from the
decomposition of the covariance matrix. As we are going to show in
\cref{sec:results} this new metric allows a better partitioning both for
regression and classification tasks by exploiting both global and local
informations. For insights about the heuristics behind it, we refer to~\cref{rmk:heuristics as distance}.

To actually find the medoids the partitioning around medoids (PAM)
algorithm~\cite{kaufman2009finding} is used. PAM uses a greedy approach after the
initial selection of the medoids, also called representative
objects. The medoids are
changed with a non-representative object, i.e. one of the remaining samples, if
it would improve the clustering quality. This iterative process of replacing the
medoids by other objects continues until the quality of the resulting clustering
cannot be improved by any replacement. \Cref{algo:kmedoids} presents this
approach with pseudo-code.

\begin{algorithm}[htb]
  \caption{K-medoids algorithm with AS metric.}
  \label{algo:kmedoids}

  \hspace*{\algorithmicindent} \textbf{Input}: set of samples $\{ x_i \}_{i=1}^N
  \in \RR^{N_F}$\\
  \hspace*{\algorithmicindent}\hspace*{35pt} number of clusters $K$ \\
  \hspace*{\algorithmicindent}\hspace*{35pt} distance metric $d$ defined in
  \cref{eq:as_norm} \\
  \hspace*{\algorithmicindent} \textbf{Output}: set of clusters $S = \{ S_j
  \}_{j=1}^K$
  \begin{algorithmic}[1]
    \State{select initial cluster medoids}
    \Repeat
    \State{assign each sample to its closest medoid using the distance
      metric $d$}
    \State{randomly select $K$ non-representative objects}
    \State{swap the medoids with the new selected objects by
      minimizing \cref{eq:medoids_criterion}}
    \Until{clustering quality converges}
  \end{algorithmic}
\end{algorithm}

\subsection{Hierarchical top-down clustering}
\label{sec:divisive}
In this section, we present hierarchical top-down
clustering~\cite{kaufman2009finding, murphy2012machine}, and exploit
the additional information from the active subspace, as done for
K-medoids. In the following sections, we refer to this technique with
the acronym HAS.

In top-down hierarchical clustering, at each iteration the considered clusters,
starting from the whole dataset, are split further and further based on some
refinement criterion, until convergence. A nice feature of hierarchical
clustering algorithms, with respect to K-means and K-medoids, is that the number
of clusters can be omitted. Moreover, by stopping at the first refinement and
forcing the total number of clusters to be the maximum number of clusters specified, HAS can be seen as a generalization of the
previous methods: for this reason, we wanted to make the implementation
consistent with K-means and K-medoids with AS induced metric as close as
possible, as shown in the numerical results in \cref{sec:results}.

Pushing further the potential of clustering algorithms applied to local
dimension reduction in the parameter space, HAS is a versatile clustering method
that takes into account the variability of the AS dimension along the parameter
space. The price paid for this is the overhead represented by the tuning of some
hyper-parameters introduced later.

\label{subsec:has implementations}
A schematic representation of the algorithm of top-down clustering is reported
in \cref{alg:top-down}. The design is straightforward and it employs a tree data
structure that assigns at each node a possible clustering of the whole dataset:
consequent refinements are represented by children nodes down until the leaves
of the tree, that represent the final clusters.

\begin{rmk}[Normalization of the clusters at each refinement iteration]
    Each cluster, at every refinement step, is normalized uniformly along
    dimensions onto the hyper-cube domain $[-1, 1]^{n}$, even if the subdomain identified by the cluster is not a hyperrectangle. Another possible choice
    for normalization is standardization, centering the samples with their mean
    and dividing them by their standard deviation.
\end{rmk}

\begin{algorithm}[htb]
  \caption{Hierarchical top-down algorithm.}
  \label{alg:top-down}

  \hspace*{\algorithmicindent} \textbf{Input}: set of samples $S =
  \{x_i \}_{i=1}^N \in \RR^{N_F}$\\
  \hspace*{\algorithmicindent}\hspace*{35pt} maximum number of clusters $K$ \\
  \hspace*{\algorithmicindent}\hspace*{35pt} range of number of children
  $\{n_{min}^{child}, n_{max}^{child}\}$ \\
  \hspace*{\algorithmicindent}\hspace*{35pt} minimum number of elements in a
  cluster $n_{el}$ \\
  \hspace*{\algorithmicindent}\hspace*{35pt} indicator for refinement
  $I$ \\
  \hspace*{\algorithmicindent}\hspace*{35pt} distance metric $d$ \\
  \hspace*{\algorithmicindent}\hspace*{35pt} minimum and maximum AS dimensions $r_{min},r_{max}$ \\
  \hspace*{\algorithmicindent}\hspace*{35pt} score tolerance
  $\epsilon$ \\
  \hspace*{\algorithmicindent} \textbf{Output}: refinement tree $T$
  \begin{algorithmic}[1]
    \State{add the initial cluster $S$ to FIFO queue $q=\{S\}$}
    \While{$q\neq\varnothing $}
    \State{take $S_{j}$, the first element
      from queue $q$}
    \State{apply the refinement function in \cref{alg:refine} to
      $S_{j}$ to get $\{S_{i}\}_{i}$}
    \State{add
      $\{S_{i}\}_{i}$ to the queue $q$}
    \If{the score tolerance
    $\epsilon$ is reached \textbf{or} other constraints are violated}
  \State{\textbf{break}}
  \EndIf
  \EndWhile
  \end{algorithmic}
\end{algorithm}

\begin{algorithm}[htb]
\caption{Refinement function.}\label{alg:refine}

\hspace*{\algorithmicindent} \textbf{Input}: cluster $S= \{ x_i
\}_{i=1}^N \in \RR^{N_F}$ \\
\hspace*{\algorithmicindent}\hspace*{35pt} number of clusters per tree
refinement level $K$\\
\hspace*{\algorithmicindent}\hspace*{35pt} range of number of children
$\{n_{min}^{child},n_{max}^{child}\}$ \\
\hspace*{\algorithmicindent}\hspace*{35pt} minmum number of elements in a
cluster $n_{el}$\\
\hspace*{\algorithmicindent}\hspace*{35pt} indicator for refinement
$I$\\
\hspace*{\algorithmicindent}\hspace*{35pt} distance metric $d$\\
\hspace*{\algorithmicindent}\hspace*{35pt} minimum and maximum AS dimensions $r_{min},r_{max}$ \\
\hspace*{\algorithmicindent} \textbf{Output}: $\{S_j\}_{j=1}^{n^{child}}$, the children of cluster $S$
  \begin{algorithmic}[1]
    \State{set best score to $b=0$}
    \For{\textbf{each} $n^{child}$ from
    $n_{min}^{child}$ to $n_{max}^{child}$}
    \State{apply the
    chosen clustering algorithm (e.g. K-medoids) with $n^{child}$
    clusters and metric $d$ to obtain the clusters
    $\{S_{j}\}_{j}^{n^{child}}$}
  \State{evaluate the estimator of the
    error $I$ for the refinement $\{S_{j}\}_{j}$, considering
    also the minimum and maximum reduced dimensions 
    $r_{min},r_{max}$}
  \If{$I > b$ \textbf{and} the
    minimum number of elements $n_{el}$ is not reached \textbf{and} the
    maximum number of clusters $K$ is not reached}
  \State{save the best
    refinement $\{S_{j}\}_{j}$ and update the best score $b$}
  \EndIf
  \EndFor
  \end{algorithmic}
\end{algorithm}

The procedure depends on many parameters that have to be tuned for the specific
case or depend \textit{a priori} on the application considered: the maximum
number of clusters ($K$), the minimum and maximum number of children nodes ($n_{min}^{child},\ n_{max}^{child}$), the
tolerance for the score on the whole domain ($\epsilon$), the minimun and maximum dimension
of the active subspace ($r_{min},r_{max}$), and the minimum number of elements ($n_{el}$) of each
cluster (usually $n_{el} > r$, where $r$ is the local AS dimension).

More importantly the method is versatile for the choice of clustering criterion,
indicator for refinement ($I$), distance metric ($d$, from equation~\eqref{eq:as_norm}) and regression method. In the following sections we
consider K-means and K-medoids with the active subspaces distance as
clustering criterion (see \cref{sec:kmedoids}), but other clustering algorithms
could in principle be applied at each refinement.

\begin{rmk}[Heuristics behind the choice of the active subspaces metric for K-medoids]
    \label{rmk:heuristics as distance}
    Having in mind that the optimal profile
    $h(\yy)=\mathbb{E}_{\boldsymbol{\mu}_{i}}[f\vert P_{r(S_i), i}](\yy)$ from
    \cref{def:local ridge approximation} is approximated as
    $h(\yy)=f(\yy)=f(P_{r}\mathbf{x})$ as reported in \cref{rmk:approx_opt_profile}, we can argue that clustering with the AS metric from
    \cref{eq:as_norm} is effective since, for this choice of the metric, the
    clusters tend to form transversally with respect to the active subspace
    directions. This is because the metric weights more the components with
    higher eigenvalues. So clustering with this metric reduces heuristically
    also the approximation error induced by the choice of the non-optimal
    profile.
\end{rmk}

Other clustering criterions employed must satisfy the subspace Poincar\'e
inequality for each cluster. Regarding the regression method we employ Gaussian
process regression with RBF-ARD kernel~\cite{williams2006gaussian}. The
procedure for response surface design with Gaussian processes and ridge
approximation with active subspaces can be found in \cite{constantine2015active,
romor2022kas}. As for the indicator for refinement (I), the local $R^2$ score in
\cref{rmk:local r2} is employed to measure the accuracy of the ridge
approximation against a validation dataset and the estimator from the RHS of
\cref{eq:indicator} is used to determine the dimension of the active subspace of
each cluster.

Here, we make some considerations about the complexity of the algorithm. For each refinement, considering an
intermediate cluster of $K$ elements, the most expensive tasks are: the active
subspace evaluation $O((K/m)np^2+(K/m)n^2p+n^3)$ (the first two costs refer to
matrix multiplications, while the third for eigendecomposition), the clustering
algorithm, for example K-medoids with AS distance $O(K(K-m)^2)$, and the
Gaussian process regression $O((K/m)^3p^3)$, where $p$ is the dimension of the
outputs and $m=n_{min}^{child}$ and $M=n_{max}^{child}$ are the minimum and maximum number of children clusters, for a more compact notation. In the worst case the height of the
refinement tree is $l=\log_{m}{N/n_{el}}$ where $n_{el}$ is the minimum number of
elements per cluster. In \cref{sec:complexity} we report the detailed
computational costs associated to each refinement level.

\section{Classification with local active subspace dimension}
\label{sec:classification}
A poor design of the parameter space could add an avoidable complexity to the
surrogate modeling algorithms. Often, in practical applications, each parameter
range is chosen independently with respect to the others. Then, it is the
responsibility of the surrogate modeling procedure to disentangle the
correlations among the parameters. However, in this way, looking at the response
surface from parameters to outputs, regions that present different degrees of
correlation are treated indistinctly. In this matter, a good practice is to
study as a preprocessing step some sensitivity measures, like the total Sobol'
indices~\cite{sullivan2015introduction} among groups of parameters, and split
the parameter space accordingly in order to avoid the use of more expensive
surrogate modeling techniques later.
Sobol' indices or the global active subspace sensitivity scores give
summary statistics on the whole domain. So in general, one could study the
parameter space more in detail, classifying nonlinearly regions with respect to
the complexity of the response surface, if there are enough samples to perform
such studies.

We introduce an effective approach to tackle the problem of classification of
the parameter space with respect to a local active subspace information. With
the latter we mean two possible alternatives.
\begin{definition}[Local active subspace dimension]
    \label{def:local as dimension}
    Given a threshold $\epsilon>0$, the pairs of inputs and gradients $\{(\XX_i,
        \mathbf{dY}_i)\}_{i}$ associated to an objective function of interest
        $f:\mathcal{X}\subset\mathbb{R}^n\rightarrow\mathbb{R}$, the size of the
        neighbourhood of sample points to consider $N\geq n$, and a subsampling
        parameter $p\in\mathbb{N},\ p\leq N$, the local active subspace
        dimension $r_i$ associated to a sample point $\XX_i\in\mathcal{X}$ is the
        positive integer
    \begin{equation*}
        r_i=\argmin_{1\leq r\leq  p} \left\{\text{tr}\left( (Id-P_r) \left(\frac{1}{p}\sum_{i\in J}
            \mathbf{dY}_i\otimes\mathbf{dY}_i\right) (Id-P_r) \right)
        \leq \epsilon \ \bigg\vert \ J\in C(N, p)\right\},
    \end{equation*}
    where $C(N, p)$ is the set of combinations without repetition of the $N$
    elements of the Euclidean neighbourhood of $\XX_i$ in $p$ classes and $P_r$ is the projection onto the first $r$ eingenvectors of the symmetric positive define matrix
    \begin{equation*}
      \frac{1}{p}\sum_{i\in J}
            \mathbf{dY}_i\otimes\mathbf{dY}_i.
  \end{equation*}
\end{definition}

\begin{definition}[Local active subspace]
    \label{def: local as}
    Given the pairs of inputs and gradients $\{(\XX_i, \mathbf{dY}_i)\}_{i}$
    associated to an objective function of interest
    $f:\mathcal{X}\subset\mathbb{R}^n\rightarrow\mathbb{R}$, the size of the
    neighbourhood  of sample points to consider $N\geq n$, and a fixed dimension
    $p\in\mathbb{N},\ 1\leq p\leq N$, the local active subspace $W_i$ associated
    to a sample point $\XX_i\in\mathcal{X}$ is the matrix of the first $p$
    eigenvectors of the spectral decomposition of
    \begin{equation}
        \frac{1}{N}\sum_{i\in U}\mathbf{dY}_i\otimes\mathbf{dY}_i,
    \end{equation}
    where $U$ is the neighbourhood  of sample points of $\XX_i$ with respect to the
    Euclidean distance. In practice, we choose $p$ close to the global active
    subspace dimension. The pairs $\{(\XX_i, W_i)\}_i$ can be thought as a
    discrete vector bundle of rank $p$ and $\{W_i\}_i$ can be thought as a
    subset of points of the Grassmannian $\text{Gr}(N, p)$, that
      is the set of $p$-dimensional subspaces in an $N$-dimensional
      vector space.
\end{definition}

Starting from the pairs of inputs-gradients $\{(\XX_i, \mathbf{dY}_i)\}_i$, the
procedure follows these steps:
\begin{enumerate}
    \item Each parameter sample is enriched with the additional feature
          corresponding to the local active subspace dimension from
          \cref{def:local as dimension} or the local active subspace from
          \cref{def: local as}, represented by the variable $\ZZ$.
    \item Each sample $\XX_i$ is labelled with an integer $l_i$ that will be used
          as classification label in the next step. To label the pairs $\{(\XX_i,
          \ZZ_i)\}_i$ we selected K-medoids with the Grassmannian metric
          \begin{equation}
            \label{eq:grassmannian distance}
              d((\XX_i, \ZZ_i), (\XX_j, \ZZ_j))= \lVert \ZZ_i - \ZZ_j \rVert_{F},
          \end{equation}
          where $\lVert\cdot\rVert_{F}$ is the Frobenius distance, in case
          $\ZZ_i$ represents the local active subspace or spectral clustering~\cite{murphy2012machine} in
          case $\ZZ_i$ is the local active subspace dimension. In the last case,
          the labels correspond to the connected components of the graph built
          on the nodes $\{(\XX_i, \ZZ_i)\}_i$ with adjacency list corresponding to
          the nearest nodes with respect to the distance
          \begin{equation}
            \label{eq: as local distance}
            d((\XX_i, \ZZ_i), (\XX_j, \ZZ_j))= \begin{cases}
                  \infty,                 & \ZZ_i\neq \ZZ_j \\
                  \lVert \XX_i - \XX_j\rVert, & \ZZ_i=\ZZ_j 
              \end{cases},
          \end{equation}
          where $\lVert\cdot\rVert$ is the Euclidean metric in $\mathbb{R}^n$.
          The connected components are obtained from the eigenvectors associated
          to the eigenvalue $0$ of the discrete Laplacian of the
          graph~\cite{murphy2012machine}. Summarizing, we employ two labelling methods: K-medoids in case $\mathbf{Z}_i$ represents the local active subspace (Definition~\ref{def: local as}) $W_i$ or spectral clustering in case $\mathbf{Z}_i$ represents the local active subspace dimension (Definition~\ref{def:local as dimension}).
    \item A classification method is applied to the inputs-labels pairs
          $\{(\XX_i, l_i)\}_i$. Generally, for our relatively simple applications
          we apply a multilayer perceptron with 1000 hidden nodes and 2 layers.
\end{enumerate}

\begin{rmk}[Grassmann distance]
    In general regarding the \cref{def: local as}, the dimension $p$ could be
    varying among samples $\XX_i$ and one could use a more general distance with
    respect to the one from \cref{eq:grassmannian distance} that can have as
    arguments two vectorial subspaces of possibly different and arbitrary large
    dimensions.
\end{rmk}

\begin{rmk}[Gradient-free active subspace]
    In general both the response surface design and the classification procedure
    above can be carried out from the pairs $\{(\XX_i, \YY_i)\}_i$ of inputs,
    outputs instead of the sets $\{(\XX_i, \mathbf{dY}_i)\}_i$ of inputs,
    gradients. In fact, the gradients $\{\mathbf{dY}_i\}$ can be approximated in
    many different ways~\cite{constantine2015active} from $\{(\XX_i, \YY_i)\}_i$.
    In the numerical results in ~\cref{sec:results} when the gradients are not
    available they are approximated with the gradients of the local
    one-dimensional polynomial regression built on top of the neighbouring
    samples.
  \end{rmk}

  \begin{algorithm}[htb]
\caption{Classification with local features from the AS information.}
\label{alg:classification}

\hspace*{\algorithmicindent} \textbf{Input}: inputs-gradients pairs $\{(\XX_i,
\mathbf{dY}_i)\}_{i\in I}$ as training dataset\\
\hspace*{\algorithmicindent}\hspace*{35pt}  local features based on AS
information $\{\ZZ_i\}_{i\in I}$ \\
\hspace*{\algorithmicindent}\hspace*{35pt} labelling method based on the
  distance $d$ from~\cref{eq:grassmannian distance} or \cref{eq: as local
  distance}\\
\hspace*{\algorithmicindent}\hspace*{35pt} classification method taking 
the inputs-labels pairs $\{(\XX_i, l_i)\}_{i\in I}$\\
\hspace*{\algorithmicindent}
\textbf{Output}: predictor for new test inputs and classes of the training
dataset.

\begin{algorithmic}[1]
  \For{\textbf{each} $i\in I$}
  \State{evaluate feature $\ZZ_i$ from $(\XX_i,
    \mathbf{dY}_i)$ and the neighbouring points of $\XX_i$}
  \EndFor
  \State{initialize the $\vert I \vert \times \vert I \vert$ distance
    matrix $M$ associated to $\{(\XX_i, \ZZ_i)\}_{i\in I}$}
  \For{\textbf{each} $i \in I$}
  \For{\textbf{each} $i\leq j \in I$}
  \State{$M(i, j)=d((\XX_i, \ZZ_i), (\XX_j,
    \ZZ_j))$}
  \EndFor
  \EndFor
  \State{use the labelling method with input $M$, to
    assign a label $l_i$ for each $(\XX_i, \ZZ_i)$}
  \State{train the
  classification method with the inputs-labels pairs $\{(\XX_i,
  l_i)\}_{i\in I}$}
\end{algorithmic}
\end{algorithm}

\section{Numerical results}
\label{sec:results}

In this section we apply the proposed localized AS method to some datasets of
increasing complexity. We emphasize that the complexity is not only
defined by the number of parameters but also by the intrinsic
dimensionality of the problem. We compare the clustering techniques presented in
\cref{sec:local}, and we show how the active subspaces-based distance metric
outperforms the Euclidean one for those functions which present a global lower
intrinsic dimensionality. We remark that for hierarchical top-down clustering we
can use both metrics, and we always show the best case for the specific dataset.

We start from a bidimensional example for which we can plot the clusters and the
regressions, and compare the different techniques. Even if it is not a case for
which one should use parameter space dimensionality reduction we think it could
be very useful for the reader to understand also visually all the proposed
techniques. For the higher dimensional examples we compare the accuracy of the
methods in terms of $R^2$ score and classification performance. All the
computations regarding AS are done with the open source Python
package\footnote{Available at \url{https://github.com/mathLab/ATHENA/}.} called
ATHENA~\cite{romor2020athena}, for the classification algorithms we use the
scikit-learn package~\cite{sklearn_api}, and for the Gaussian process regression
GPy~\cite{gpy2014}.

We suppose the domain $\mathcal{X}$ to be a
$n$-dimensional hyperrectangle. we are going to rescale the input
parameters $\XX$ to $[-1, 1]^n$.

\subsection{Some illustrative bidimensional examples}
We start by presenting two bidimensional test cases  to show
every aspect of the methodology together with illustrative plots. First
we analyse a case where a global active subspace, even if present, does not
provide a regression accurate enough along the active direction, in
\cref{sec:quartic}. Then we consider a radial symmetric function for which, by
construction, an AS does not exist, in \cref{sec:cosine}, and the use of K-means
is instead preferable since we cannot exploit a privileged direction in the
input domain.

\subsubsection{Quartic function}
\label{sec:quartic}

Let us consider the following bidimensional quartic function $f(\xx) = x_1^4 -
x_2^4$, with $\xx = (x_1, x_2) \in [0, 1]^2$. In \cref{fig:quartic_as_global} we
can see the contour plot of the function, the active subspace direction ---
translated for illustrative reasons ---  and the corresponding sufficient
summary plot of the global active subspace, computed using $400$ uniformly
distributed samples. With sufficient summary plot we intend $f(\xx)$ plotted
against the input parameters projected onto the active subspace, that is $W_1^T
\xx$. It is clear how, in this case, a univariate regression does not produce any
useful prediction capability.

\begin{figure}[h]
\centering
\includegraphics[width=.95\textwidth]{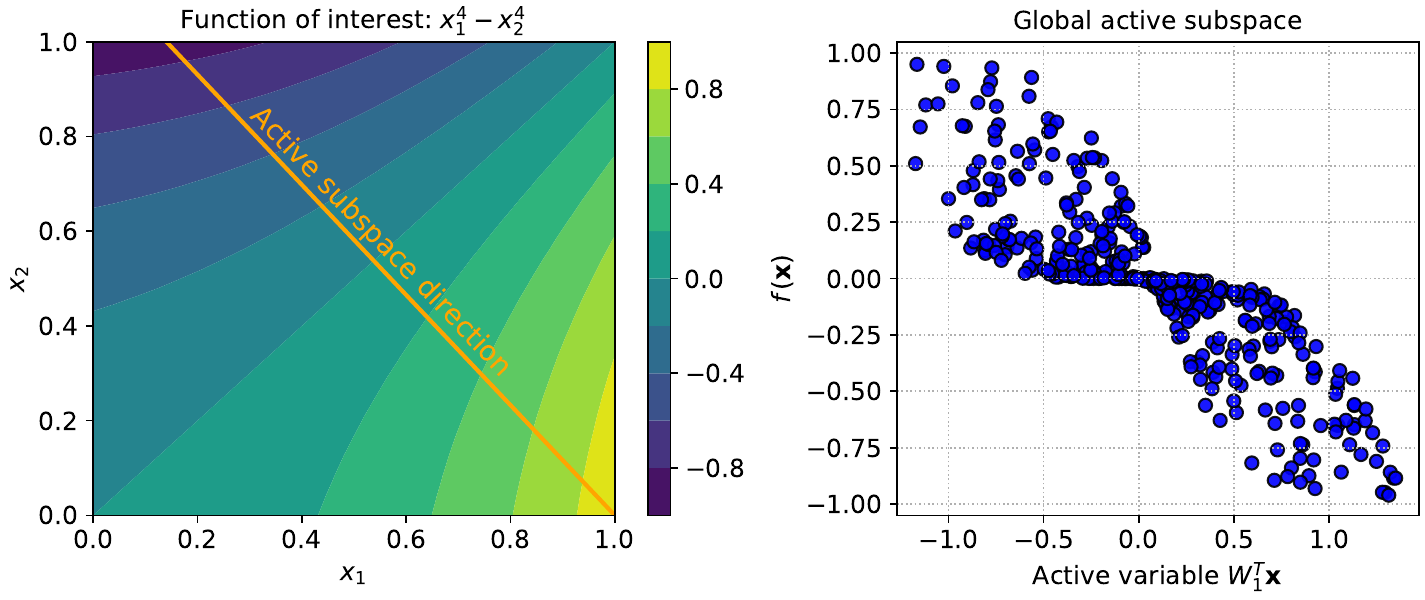}
\caption{On the left panel the contour plot of the quartic function and in
  orange the global active subspace direction. On the right panel the sufficient
  summary plot resulting projecting the data onto the global
  AS.\label{fig:quartic_as_global}}
\end{figure}

Let us apply the clustering techniques introduced in the previous sections
fixing the number of clusters to $4$. In \cref{fig:quartic_4_clusters} we can
clearly see how the supervised distance metric in \cref{eq:as_norm} acts in
dividing the input parameters. On the left panel we apply K-means which clusters
the data into $4$ uniform quadrants, while in the middle and right panels we
have K-medoids and hierarchical top-down, respectively, with a subdivision
aligned with the global AS. We notice that for this simple case the new metric
induces an identical clustering of the data. In \cref{fig:local_ssp_quartic} we
plotted the sufficient summary plots for each of the clusters individuated by
K-medoids or hierarchical top-down in
\cref{fig:quartic_4_clusters}. By using a single univariate regression
for each cluster the $R^2$ score 
improves a lot with respect to a global approach (see right panel of
\cref{fig:quartic_as_global}).

\begin{figure}[h]
\centering
\includegraphics[width=1.\textwidth]{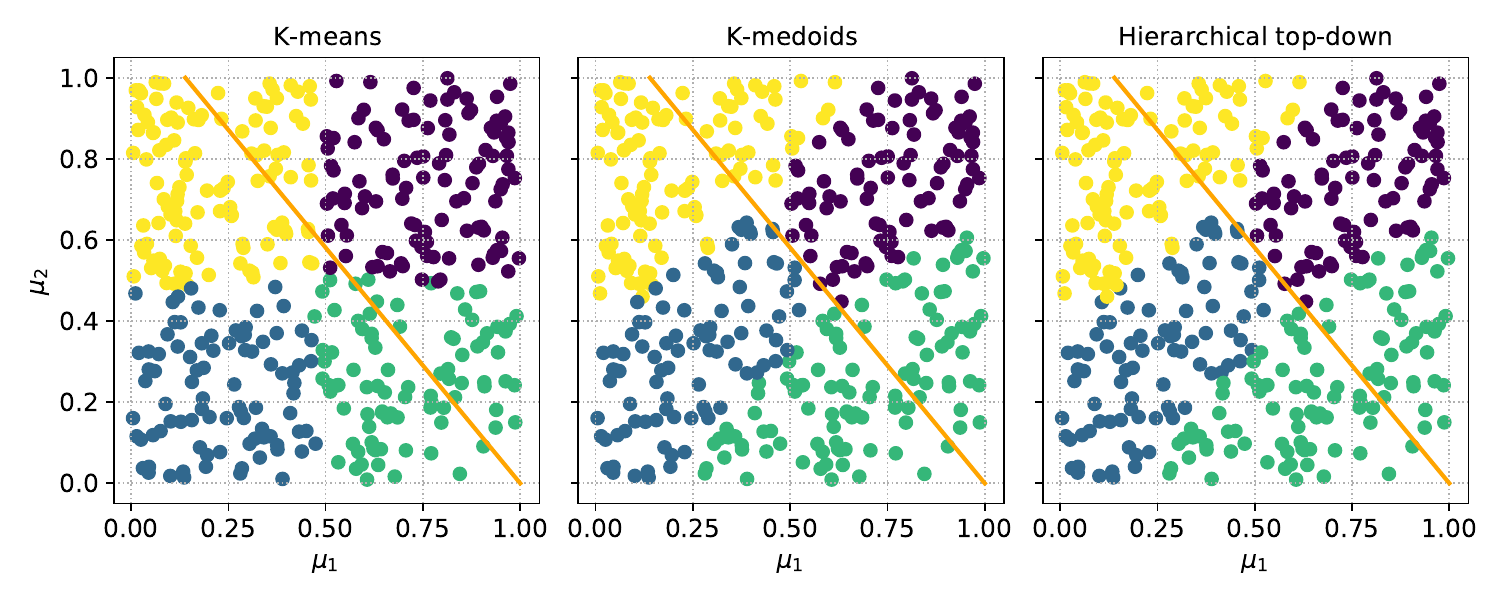}
\caption{Comparison between the different clusters obtained by K-means (on the
  left), K-medoids (middle panel), and hierarchical top-down (on the right) with
  AS induced distance metric defined in \cref{eq:as_norm} for the quartic test
  function. In orange the global active subspace direction. Every cluster is
  depicted in a different color.\label{fig:quartic_4_clusters}}
\end{figure}

\begin{figure}[h]
\centering
\includegraphics[width=1.\textwidth]{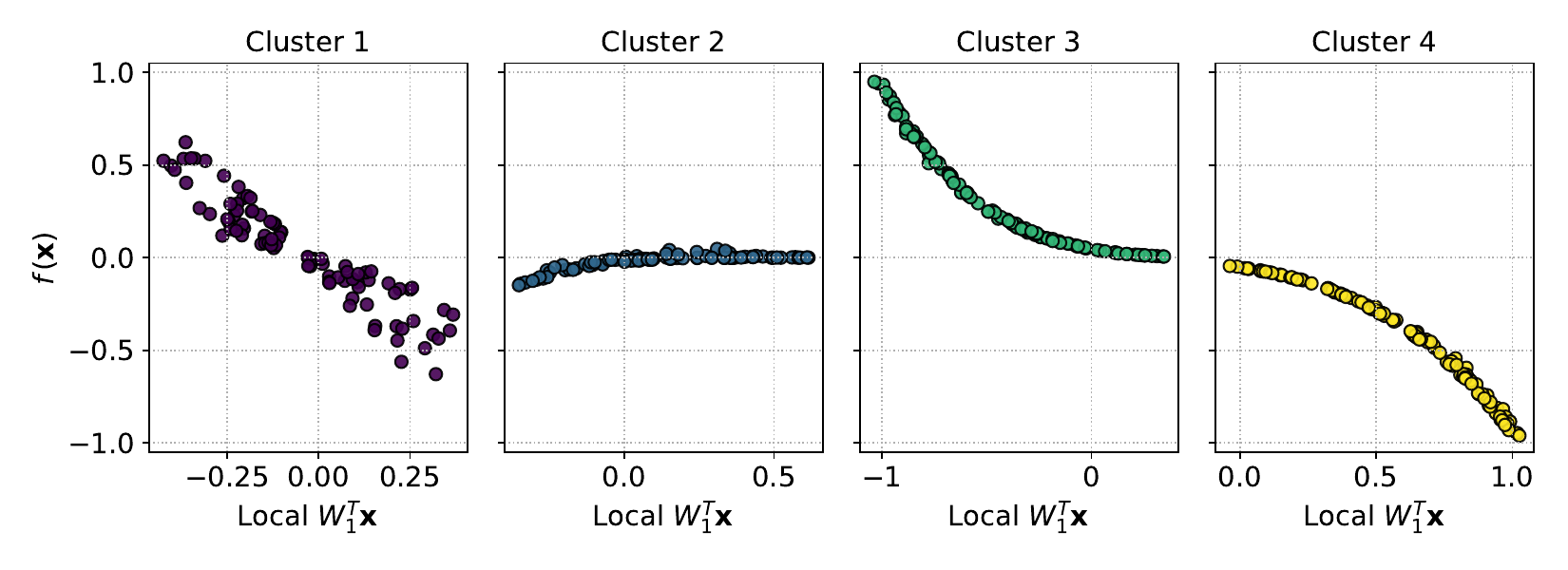}
\caption{Local sufficient summary plots for the $4$ clusters individuated by
  K-medoids or hierarchical top-down in \cref{fig:quartic_4_clusters} (colors
  correspond).\label{fig:local_ssp_quartic}}
\end{figure}

We can also compare the $R^2$ scores for all the methods, using a test datasets
of $600$ samples. In \cref{fig:quartic_r2} we report the scores for K-means,
K-medoids and for hierarchical top-down with AS-based distance metric. The score
for the global AS, which is $0.78$, is not reported in \cref{fig:quartic_r2} for
illustrative reasons. The results are very similar due to the relatively simple
test case, but we can see that even with $2$ clusters the gain in accuracy is
around $23\%$ using the metric in \cref{eq:as_norm}.

The hierarchical top-down clustering method was ran with the following
hyper-parameters: the total number of clusters is increasing from $2$ to $10$,
the minimum number of children equal to the maximum number of children equal to
$3$, uniform normalization of the clusters, the minimum size of each cluster is
$10$ elements, the clustering method is K-medoids with AS distance, the maximum
active subspace dimension is $1$.

\begin{figure}[h]
\centering
\includegraphics[width=.83\textwidth]{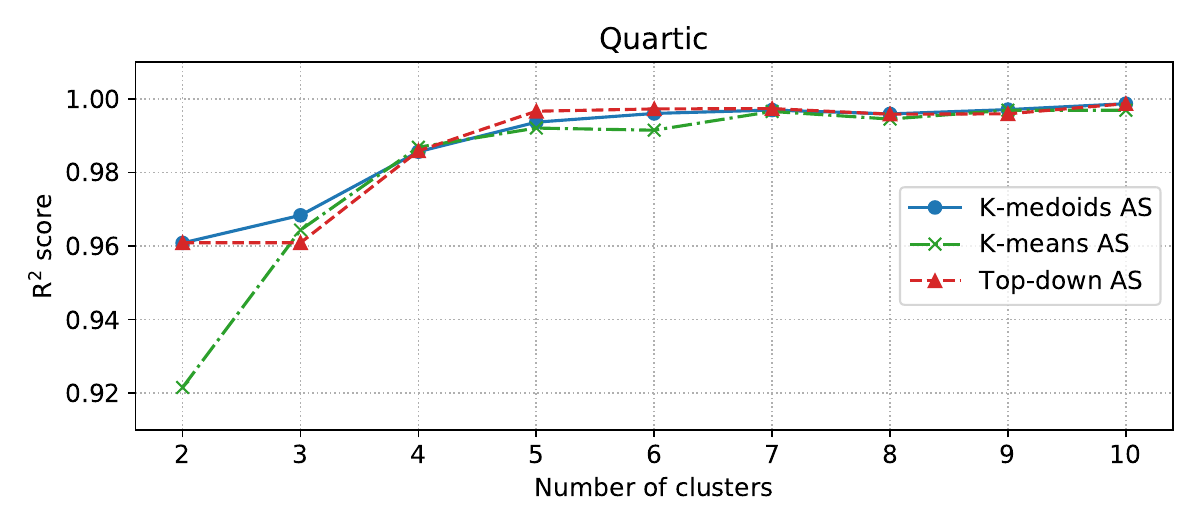}
\caption{$R^2$ scores comparison between local versions varying the number of
  clusters for the quartic function. Global AS has a score equal to
  $0.78$.\label{fig:quartic_r2}}
\end{figure}

Then we want to increase the accuracy of the regression for a fixed
number of clusters equal to $3$,
loosing in some regions the reduction in the parameter space. Starting from the
clustering with hierarchical top-down and $3$ clusters of dimension $1$, the AS
dimension of each of the 3 clusters is increased if the threshold of $0.95$ on
the local $R^2$ score is not met. In general, the local $R^2$ score is evaluated
on a validation set, for which predictions from the local response surfaces are
obtained, after each validation sample is classified into one of the $3$
clusters.

\begin{figure}
\centering
\includegraphics[width=.95\textwidth]{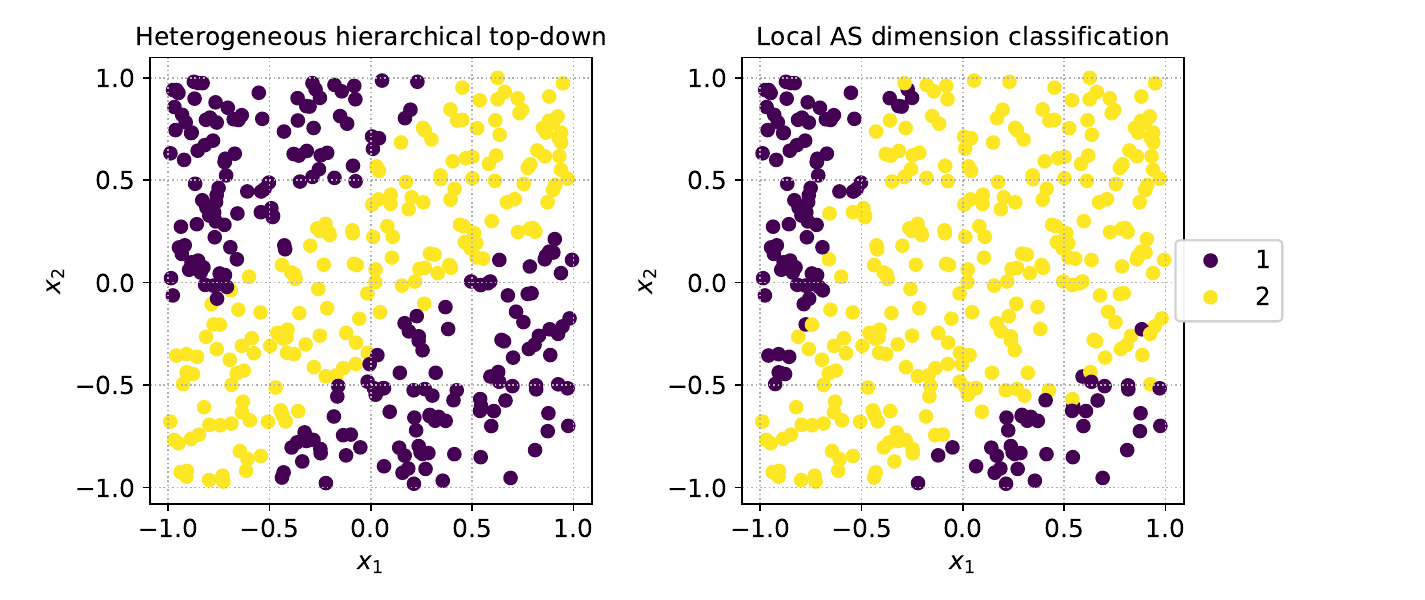}
\caption{On the left panel the hierarchical top-down clustering with
heterogeneous AS dimension and $R^2$ score equal to $1$. On the right panel the
labels of the local AS dimension from \cref{def:local as
dimension}.}\label{fig:quartic_classification_local_as}
\end{figure}

The $3$ clusters are reported in \cref{fig:quartic_classification_local_as} on
the left. The $R^2$ score on the test set is $1$, instead of around $0.97$ from
\cref{fig:quartic_r2}. To obtain this result, the central cluster AS dimension
is increased from $1$ to $2$. We compare the clustering with respect to the
classification of the local AS dimension with \cref{alg:classification} using as
features the local AS dimension as defined in \cref{def:local as dimension}, on
the right of \cref{fig:quartic_classification_local_as}. Actually,
\cref{alg:classification} is stopped after the plotted labels are obtained as
the connected components of the underlying graph to which spectral clustering is
applied: no classification method is employed, yet. It can be seen that
hierarchical top-down clustering with heterogeneous AS dimension is more
efficient with respect to the classes of \cref{alg:classification}, regarding
the number of samples associated to a response surface of dimension $2$.

\subsubsection{Radial symmetric cosine}
\label{sec:cosine}
This example addresses the case for which an active subspace is not present.
This is due to the fact that there are no preferred directions in the input
domain since the function $f$ has a radial symmetry.
For this case the exploitation of the supervised distance metric does not
provide any significant gain and K-means clustering works better on average,
since it does not use the global AS structure. The model function we consider is
$f(\xx) = \cos (\| \xx \|^2)$, with $\xx \in [-3, 3]^2$.


In \cref{fig:isotropic_r2} we compare the $R^2$ scores for K-means, K-medoids
with AS-based metric, and hierarchical top-down with Euclidean metric. We used
$500$ training samples and $500$ test samples. We see K-medoids has not a clear
behaviour with respect to the number of clusters, while the other methods
present a monotonic trend and better results on average, especially K-means. On
the other hand local models improve the accuracy considerably, even for a small
number of clusters, with respect to a global model.
\begin{figure}[h]
\centering
\includegraphics[width=.83\textwidth]{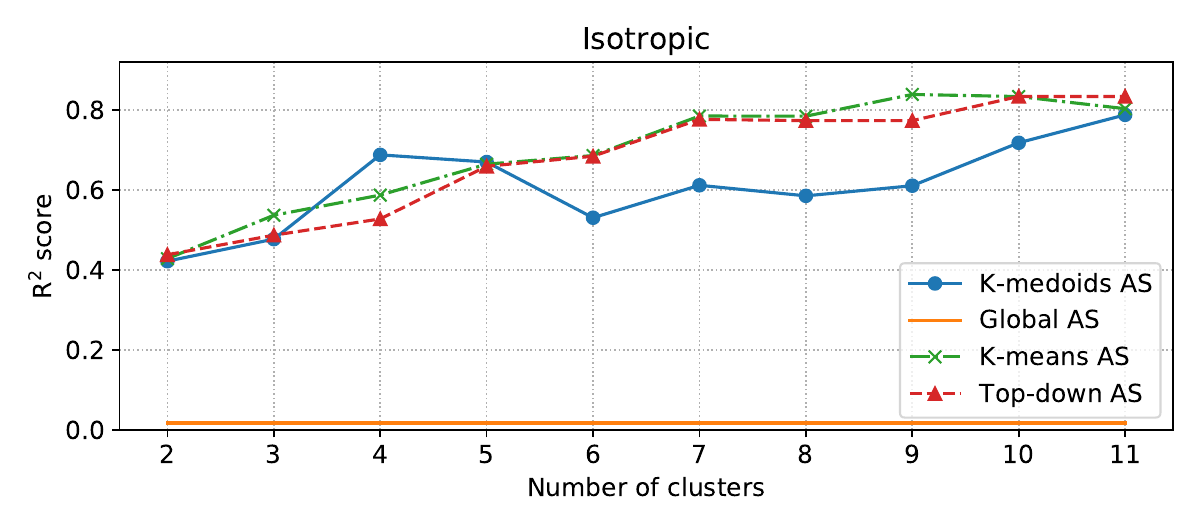}
\caption{$R^2$ scores comparison between global AS and local versions varying
  the number of clusters for the isotropic model function. Global AS corresponds
  to no clustering.\label{fig:isotropic_r2}}
\end{figure}
In this case the specifics of hierarchical top-down clustering are: the minimum
number of children is equal to the maximum, the minimum number of elements per
cluster is $10$, the clustering method chosen is K-means, the normalization
employed it the uniform one, and the total number of clusters is increasing from
$2$ to $11$.

\subsection{Higher-dimensional datasets}
In this section we consider more interesting benchmarks, for which
dimension reduction in the parameter space is useful since the starting
dimension of the parameter space is higher. We test the classification procedure
in \cref{alg:classification} with an objective function with $6$ parameters and
defined piecewise as a paraboloid with different AS dimensions. We also test the
procedure of response surface design with local AS, with a classical
$8$-dimensional epidemic benchmark model.

\subsubsection{Multi-dimensional hyper-paraboloid}
\label{sec:hyper-paraboloid}
The objective function $f:[-4, 4]^6\rightarrow\mathbb{R}$ we consider is defined
piecewise as follows

\begin{equation}
  f(x) = \begin{cases}
    x_1^2 & \text{if}\ x_1 > 0 \ \text{and}\ x_2 > 0, \\
    x_1^2+x_2^2 & \text{if}\ x_1 < 0 \ \text{and}\ x_2 > 0, \\
    x_1^2+x_2^2+x_3^2 & \text{if}\ x_1 > 0 \ \text{and}\ x_2 < 0, \\
    x_1^2+x_2^2+x_3^2+x_4^2 & \text{if}\ x_1 < 0 \ \text{and}\ x_2 < 0.
  \end{cases}
\end{equation}

In the $4$ domains in which $f$ is defined differently, we expect an AS
dimension ranging from $1$ to $4$, respectively. We employed
\cref{alg:classification} using the local AS dimensions as additional features,
from  \cref{def:local as dimension}: the values of the hyper-parameters are the
following: $\epsilon=0.999$, $N=6$, $p=4$. In \cref{fig:accuracy_study} we plot
the accuracy of the classification of the labels, associated to the connected
components of the graph built as described in \cref{alg:classification}, and
also the accuracy of the classification of the local active subspace dimension,
that takes the values from $1$ to $4$. The test dataset for both the
classification errors has size $1000$. The score chosen to asses the quality of
the classification is the mean accuracy, that is the number of correctly
predicted labels over the total number of labels. For both the classification
tasks $100$ train samples are enough to achieve a mean accuracy above $80\%$.

\begin{figure}[h]
  \centering
  \includegraphics[width=.83\textwidth]{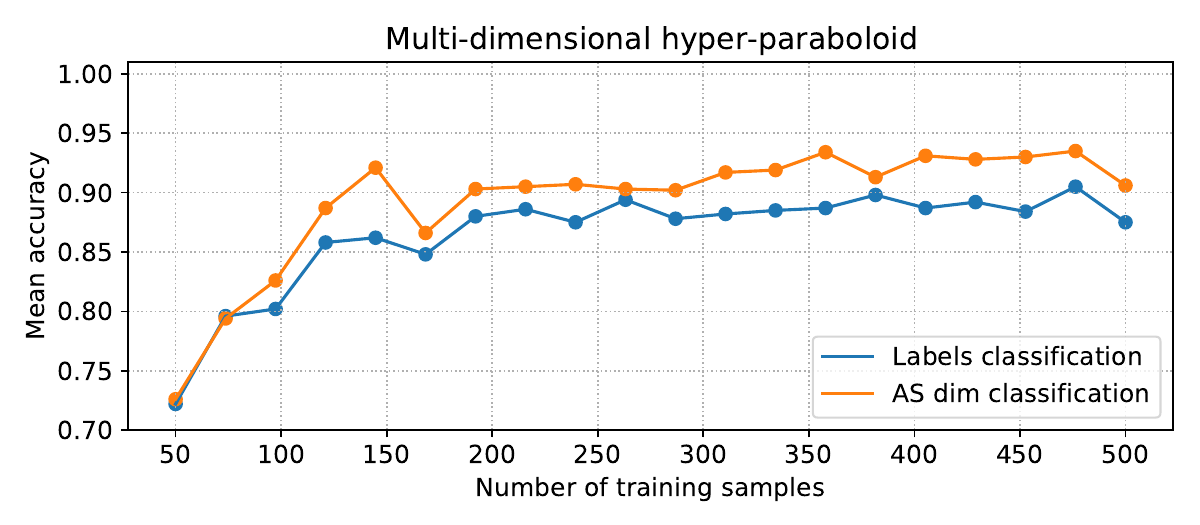}
  \caption{Mean accuracy study for a training dataset increasing in size from
  $50$ to $500$ samples. The test set is made of $1000$ independent samples. The
  classification accuracy for the procedures of connected component
  classification (in blue) and local AS dimension classification (in orange) are
  both shown.\label{fig:accuracy_study}}
\end{figure}

We remark that every step is applied to a dataset of samples in a parameter
space of dimension $6$, even if, to get a qualitative idea of the performances
of the method, in \cref{fig:qualitative plot classification} we show only the
first two components of the decision boundaries of the $4$ classes for both the
previously described classification problems.

\begin{figure}[h]
  \centering
  \includegraphics[width=0.49\textwidth]{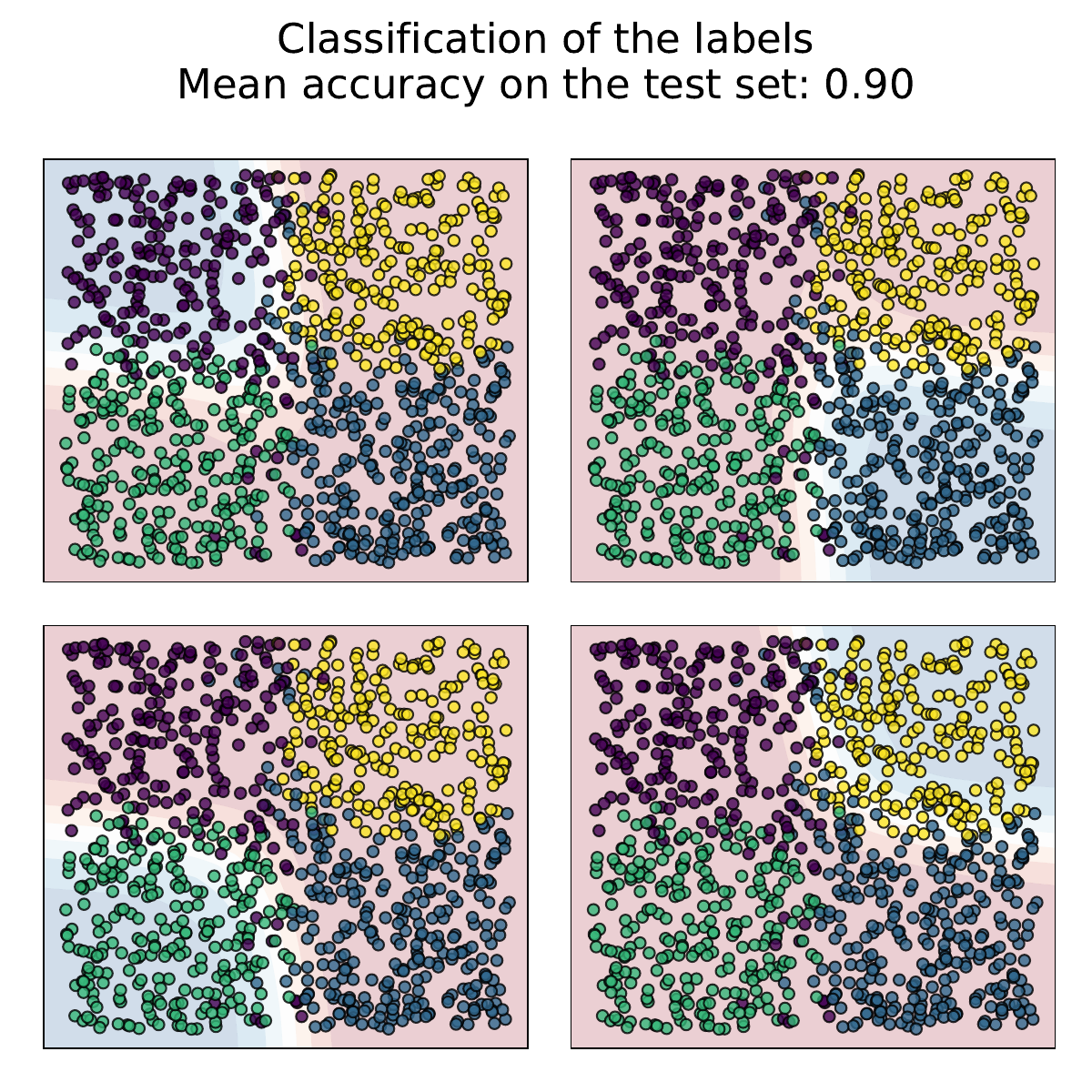}\hfill
  \includegraphics[width=0.49\textwidth]{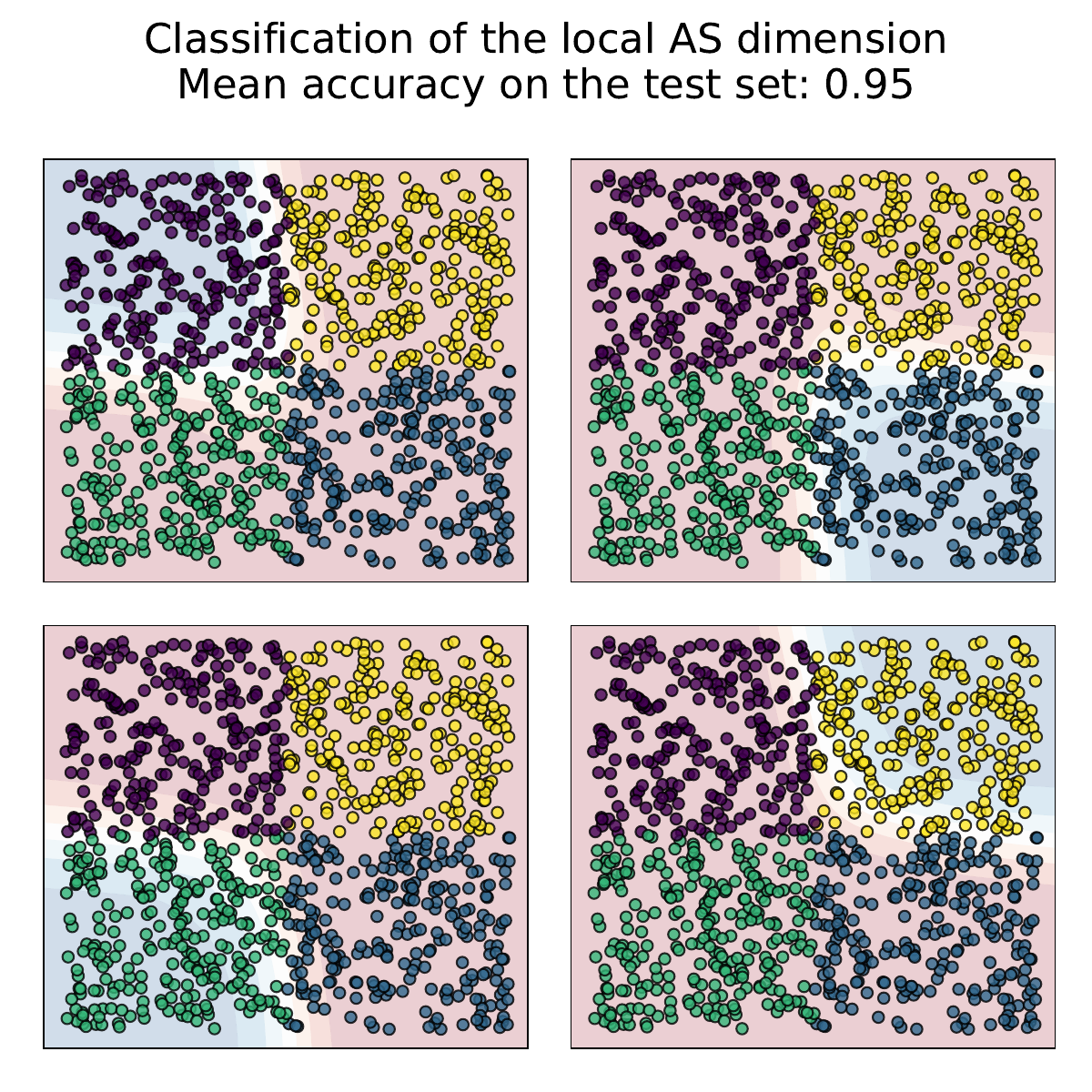}
  \caption{On the left panel, the decision boundaries of the $4$ classes
  associated to the connected components of the graph built as described in
  \cref{alg:classification}. On the right panel, the decision boundaries of the 4
  classes associated to the local AS dimension from $1$ to $4$. The datasets has
  dimension $6$, only the first two components of the decision boundaries and of
  the test samples are plotted.\label{fig:qualitative plot classification}}
\end{figure}

\subsubsection{Ebola epidemic model}
\label{sec:ebola}
In this section we examine the performance of the proposed methods over the
dataset created with the SEIR model for the spread of Ebola\footnote{The dataset
was taken from \url{https://github.com/paulcon/as-data-sets}.}. The output of
interest in this case is the basic reproduction number $R_0$ of the SEIR model,
described in~\cite{diaz2018modified}, which is computed using $8$ parameters as
follows
\begin{equation}
\label{eq:ebola}
  R_0 =\frac{\beta_1 +\frac{\beta_2\rho_1 \gamma_1}{\omega} +
  \frac{\beta_3}{\gamma_2} \psi}{\gamma_1+ \psi}.
\end{equation}
As shown in previous works, this function has a lower intrinsic dimensionality,
and thus a meaningful active subspace, in particular of dimension $1$. To
evaluate the performance of the local AS we compute the $R^2$ score, as in
\cref{eq:r2_def}, varying the number of clusters from $2$ to $10$ for all the
methods presented. The test and training datasets are composed by $500$ and
$300$, respectively, uniformly distributed and independent samples. The results
are reported in \cref{fig:ebola_r2}, where as baseline we reported the $R^2$ for
the GPR over the global AS. We can see how the use of the AS-based distance
metric contributes the most with respect to the actual clustering method
(compare K-medoids and hierarchical top-down in the plot). K-means, instead,
does not guarantee an improved accuracy (for $4$ and $9$ clusters), and in
general the gain is limited with respect to the other methods, especially for a
small number of clusters which is the most common case in practice, since
usually we work in a data scarcity regime. The results for K-medoids and
top-down are remarkable even for a small amount of clusters with an $R^2$ above
$0.9$ and an improvement over $10\%$ with respect to the global AS, which means
that no clustering has been used.

The hyper-parameters for the hierarchical top-down algorithm are the following:
the maximum local active subspace dimension is $1$, the maximum number of
children is equal to the number of total clusters, the minimum number of
children is $2$ at each refinement level, the minimum number of elements per
cluster is $10$, and the clustering method for each refinement is K-medoids with
AS distance.

\begin{figure}[h]
\centering
\includegraphics[width=.83\textwidth]{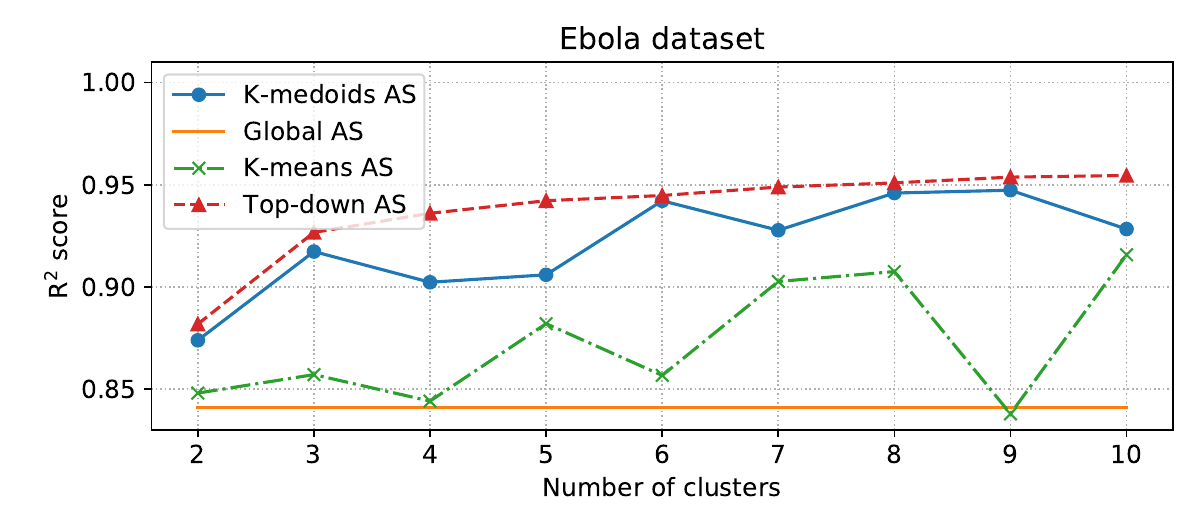}
\caption{$R^2$ scores comparison between global AS and local versions varying
  the number of clusters for the Ebola spread model. Global AS corresponds to no
  clustering.\label{fig:ebola_r2}}
\end{figure}

\subsection{Datasets with vectorial outputs}
In this section we want to show how hierarchical top-down clustering and the
classification procedure of \cref{alg:classification} can be combined to improve
the overall reduction in the parameter space, for a fixed lower threshold in the
$R^2$ score. For the response surface design with active subspaces for vectorial
outputs we refer to~\cite{zahm2020gradient, romor2022kas}.

\subsubsection{Poisson equation with random diffusivity}
\label{sec:poisson}
Let us consider the stochastic Poisson problem on the square $\xx=(x,
y)\in\Omega := [0, 1]^2$, defined as:
\begin{equation}
\label{eq:sPDE}
\begin{cases}
-\nabla\cdot (\kappa\ \nabla u)=1, & \xx \in\Omega, \\
u = 0, & \xx\in\partial\Omega_{\text{top}}\cup\partial\Omega_{\text{bottom}},\\
u = 10 y(1-y), &\xx \in\partial\Omega_{\text{left}},\\
\mathbf{n}\cdot\nabla u = 0, & \xx \in\partial\Omega_{\text{right}},
\end{cases}
\end{equation}
with homogeneous Neumann boundary condition on $\partial\Omega_{\text{right}}$,
and Dirichlet boundary conditions on the remaining part of $\partial\Omega$. The
diffusion coefficient $\kappa:(\Omega, \mathcal{A}, P)\times\Omega\rightarrow
\RR$, with $\mathcal{A}$ denoting a $\sigma$-algebra, is such that $\log(\kappa)$
is a Gaussian random field, with covariance function $G(\xx,\yy)$ defined by
\begin{equation}
G(\xx, \yy) = \exp\left(-\frac{\lVert \xx - \yy \rVert^{2}}{\beta^{2}}
\right),\quad \forall \, \xx,\yy\in\Omega,
\end{equation}
where the correlation length is $\beta=0.03$. We approximate this random field
with the truncated Karhunen–Lo\`eve decomposition as
\begin{equation}
\kappa(s, \xx) \approx \exp\left(\sum_{i=1}^m X_i(s)
  \gamma_{i} \boldsymbol{\psi}_i (\xx) \right) , \qquad\forall (s,
\xx) \in \Omega\times\Omega,
\end{equation}
 where $(X_{i})_{i\in 1,\dots, m}$ are independent standard normal distributed
 random variables, and the eigenpairs of the Karhunen–Lo\`eve decomposition of
 the zero-mean random field $\kappa$ are denoted with $(\gamma_{i},
 \boldsymbol{\psi}_{i})_{i\in 1,\dots, m}$. The parameters $(X_{i})_{i\in
 1,\dots, m=10}$ sampled from a standard normal distribution are the
 coefficients of the Karhunen-Lo\`eve expansion, truncated at the first $10$
 modes, so the parameter space has dimension $m=10$.

The domain $\Omega$ is discretized with a triangular unstructured mesh
$\mathcal{T}$ with $3194$ triangles. The simulations are carried out with the
finite element method with polynomial order $1$. The solution $u$ is evaluated
at $1668$ degrees of freedom, thus the output is vectorial with dimension
$d=1668$. As done in~\cite{zahm2020gradient, romor2022kas}, the output is
enriched with the metric induced by the Sobolev space $H^1(\Omega)$ on to the
finite element space of polynomial order $1$: the metric is thus represented by
a $d\times d$ matrix $M$ obtained as the sum of the mass and stiffness matrices
of the numerical scheme and it is involved in the AS procedure when computing
the correlation matrix $\mathbb{E}\left [ Df\ M\ Df^{T} \right ]$, where $Df$ is
the $m\times d$ Jacobian matrix of the objective function $f:\RR^{10}\to \RR^d$, that maps the first $m=10$ coefficients of the Karhunen-Lo\'eve expansion $(X_i)_{i\in 1,\dots, m}$ to the solution $u$. The Jacobian matrix is evaluated for each set of parametric instances with the adjoint method, as in~\cite{romor2022kas}.

Since the output is high-dimensional we classified with
\cref{alg:classification} the output space in $6$ clusters, using the Grassmann
distance from \cref{eq:grassmannian distance}, as shown in
\cref{fig:clusters_spoisson}.

\begin{figure}[h]
  \centering
  \includegraphics[width=.6\textwidth]{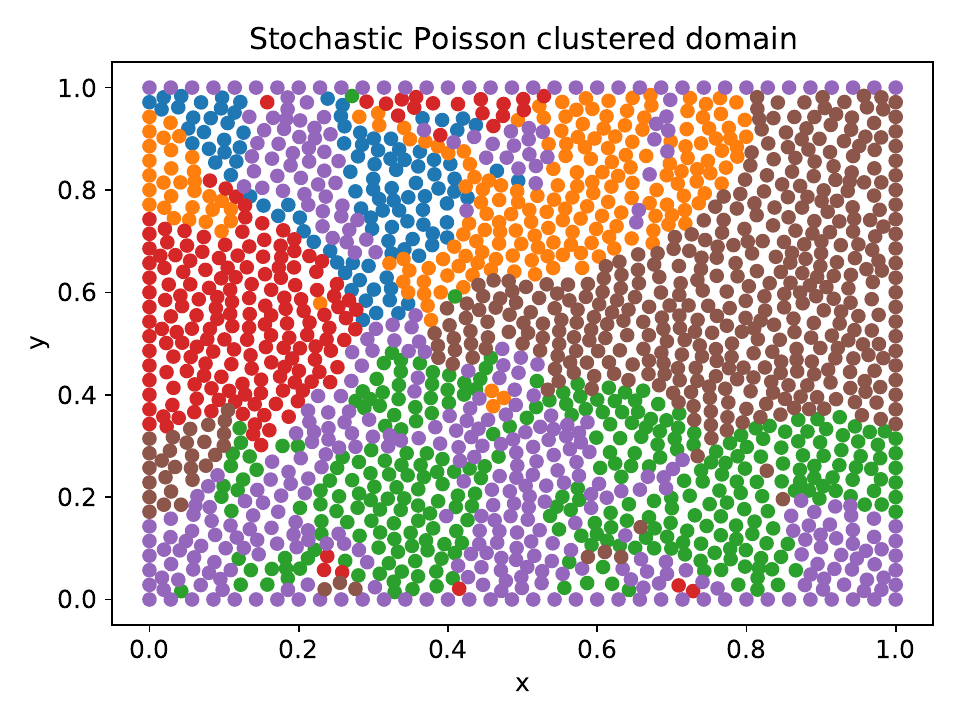}
  \caption{Subdivision of the spatial domain $\Omega$ in $6$ clusters based on
  the Grassmann distance from \cref{def: local as}, i.e. the clusters correspond
  to the connected components of the graph built on top of the degrees of
  freedom with adjacency list determined using as distance \cref{def: local
  as}.\label{fig:clusters_spoisson}}
  \end{figure}

Afterwards we applied hierarchical top-down clustering to every one of the $6$
triplets of inputs-outputs-gradients, obtained restricting the outputs and the
gradients to each one of the $6$ clusters. The specifics of hierarchical
top-down clustering we employed are the following: the minimum and maximum
number of children for each refinement are equal to the total number of
clusters, which is $4$, the minimum number of elements in each cluster is $10$,
and the clustering algorithm chosen is K-medoids with the AS distance. The size
of the training and test datasets is respectively of $500$ and $150$. The
gradients are evaluated with the adjoint method. Since the output is vectorial
we employed the mean $R^2$ score, where the average is made among the components
of the vectorial output considered.

Then for every lower threshold on the $R^2$ score we increase one by one the
dimension of the $6\times 4$ local clusters, until all the $R^2$ scores of each
of the $6$ triplets are above the fixed threshold. The same procedure is applied
to the whole dataset of inputs-outputs-gradients but executing hierarchical
top-down clustering just once, for all the output's components altogether.

\begin{figure}[h]
\centering
\includegraphics[width=.83\textwidth]{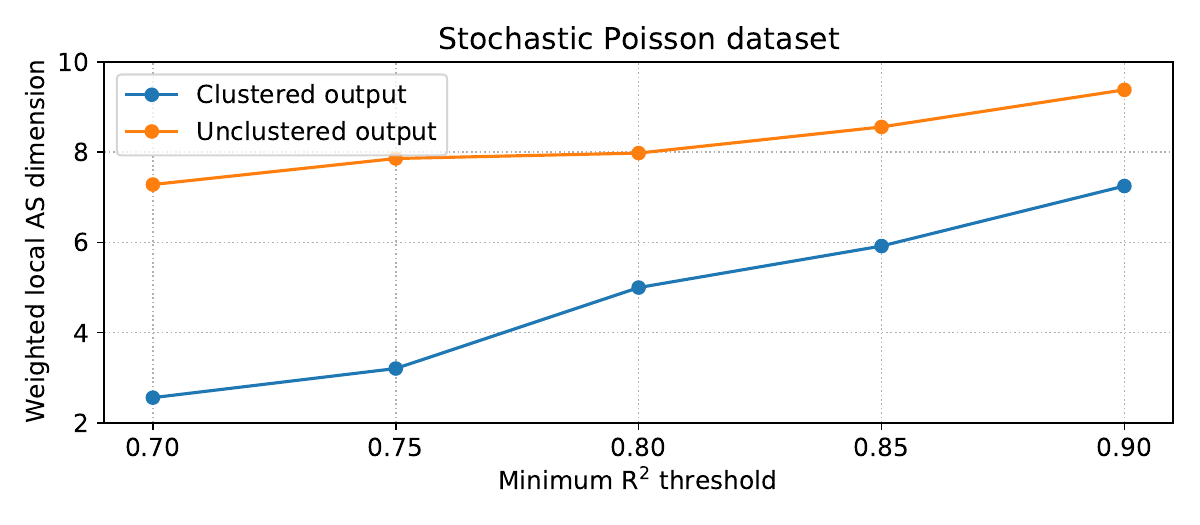}
\caption{In orange the local AS dimensions weighted on the number of elements of
each of the $4$ clusters in the parameter space, obtained with hierarchical
top-down clustering. In blue the local AS dimensions weighted on the number of
elements of each of the $4$ clusters in the parameter space, obtained with
hierarchical top-down clustering, times $6$ clustered outputs (see
\cref{fig:clusters_spoisson}) for a total of $24$ terms in the weighted
average.\label{fig:laplace_r2}}
\end{figure}

The results are reported in \cref{fig:laplace_r2}. In the case of the clustered
outputs, the local dimension of each one of the $6$ clustered outputs times $4$
local clusters in the parameter space, for a total of $24$ local clusters, are
weighted with the number of elements of each cluster. In the same way the $4$
clusters of the case with unclustered outputs is weighted with the number of the
elements of each one of the $4$ clusters. It can be seen that for every fixed
threshold, there is an evident gain, with respect to the dimension reduction in
the parameter space, in clustering the outputs and then performing hierarchical
top-down clustering in the parameter space.

\subsubsection{Shape design of an airfoil}
\label{sec:naca}
For this vectorial test case we consider the temporal evolution of the lift
coefficient of a parametrized NACA airfoil. Here we briefly present the problem
we solve to create the dataset, we refer to~\cite{tezzele2020enhancing} for a
deeper description.

Let us consider the unsteady incompressible Navier-Stokes equations described in
an Eulerian framework on a parametrized space-time domain $S(\mupar) =
\Omega(\mupar) \times [0,T] \subset \RR^2\times\RR^+$. The vectorial velocity
field $\mathbf{u}: S(\mupar) \to \RR^2$, and the scalar pressure field $p:
S(\mupar) \to \RR$ solve the following parametric PDE:
\begin{equation}
\label{eq:mortech_navstokes}
\begin{cases}
\mathbf{u_t}+ \nabla \cdot (\mathbf{u} \otimes \mathbf{u})- \nabla \cdot
2 \nu \mathbf{\nabla}^s \mathbf{u} = - \nabla p &\mbox{ in } S(\mupar),\\
\nabla \cdot \mathbf{u} = \mathbf{0} &\mbox{ in } S(\mupar),\\
\mathbf{u} (t,\xx) = \mathbf{f}(\xx) &\mbox{ on } \Gamma_{\text{in}} \times [0,T],\\
\mathbf{u} (t,\xx) = \mathbf{0} &\mbox{ on } \Gamma_{0}(\mupar) \times [0,T],\\
(\nu \nabla \mathbf{u} - p \mathbf{I} ) \mathbf{n} = \mathbf{0} &\mbox{ on } \Gamma_{\text{out}} \times [0,T],\\
\mathbf{u}(0,\xx) = \mathbf{k}(\xx) &\mbox{ in } S(\mupar)_0\\
\end{cases}.
\end{equation}

Here, $\Gamma = \Gamma_{\text{in}} \cup \Gamma_{\text{out}} \cup \Gamma_{0}$
denotes the boundary of $\Omega(\mupar)$ composed by inlet boundary, outlet
boundary, and physical walls, respectively. With $\mathbf{f}(\xx)$ we indicate
the stationary non-homogeneous boundary condition, and with $\mathbf{k}(\xx)$ the
initial condition for the velocity at $t=0$. The geometrical deformation are
applied to the boundary $\Gamma_0(\mupar)$. The undeformed configuration
corresponds to the NACA~4412 wing profile~\cite{abbott2012theory,
jacobs1933characteristics}. To alter such geometry, we adopt the shape
parametrization and morphing technique proposed in~\cite{hicks1978wing}, where
$5$ shape functions are added to the airfoil profiles. They are
commonly called Hicks-Henne bump functions. Let $y_u$ and $y_l$ be
the upper and lower ordinates of the profile, respectively. The deformation of
such coordinates is described as follows
\begin{equation}
  y_u = \overline{y_u} + \sum_{i=1}^{5} c_i r_i , \qquad
  y_l = \overline{y_l} - \sum_{i=1}^{5} d_i r_i ,
\end{equation}
where the bar denotes the reference undeformed state. The parameters $\mupar \in
\mathbb{D} \subset \mathbb{R}^{10}$ are the weights coefficients, $c_i$ and
$d_i$,  associated with the shape functions $r_i$. In particular we set
$\mathbb{D} := [0, 0.03]^{10}$. The explicit formulation of the shape functions
can be found in~\cite{hicks1978wing}. For this datasets, the Reynolds number is
$Re=50000$. The time step is $dt=10^{-3}$~s. For other specifics regarding the
solver employed and the numerical method adopted
we refer to~\cite{tezzele2020enhancing}.

As outputs we considered the values of the lift coefficient, every $15$ time
steps from $100$~ms to $30000$~ms, for a total of $1994$ components. Even in this
case the output is classified with \cref{alg:classification} with distance
defined in \cref{def:local as dimension}. The values of the lift coefficient
physically interesting are collected at last, after an initialization phase.
Nonetheless for the purpose of having a vectorial output we considered its value
from the time instant $100$~ms. The procedure finds two classes and splits the
ordered output components in two parts: from the component $0$ to $996$, the
local AS dimension is 1, for the remaining time steps it is higher. So we can
expect an improvement on the efficiency of the reduction in the parameter space
when considering separately these two sets of outputs components as
\cref{fig:mortech_r2} shows. The weighted local AS dimension is in fact lower
when using clustering, for every minimum $R^2$ threshold.

\begin{figure}[h]
\centering
\includegraphics[width=.83\textwidth]{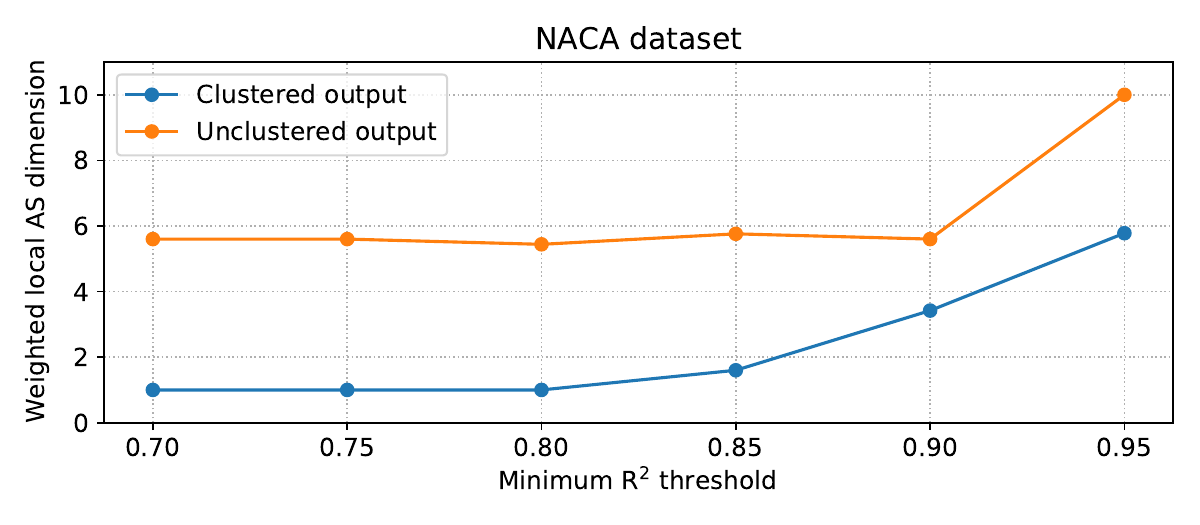}
\caption{In orange the local AS dimensions weighted on the number of elements of
each of the $2$ clusters in the parameter space, obtained with hierarchical
top-down clustering. In blue the local AS dimensions weighted on the number of
elements of each of the $2$ clusters in the parameter space, obtained with
hierarchical top-down clustering, times $2$ clustered outputs for a total of $4$
terms in the weighted average.\label{fig:mortech_r2}}
\end{figure}

\section{Conclusions and perspectives}
\label{sec:the_end}
In this work we present a new local approach for parameter space reduction which
exploits supervised clustering techniques, such as K-means, K-medoids, and
hierarchical top-down, with a distance metric based on active subspaces. We
called this method local active subspaces (LAS). The proposed metric
tend to form the clusters transversally with respect to the active
subspace directions thus reducing the approximation error induced by
the choice of the non-optimal profile.

The theoretical formulation provides error estimates for the construction of
response surfaces over the local active subspaces. We also present a
classification approach to capture the optimal AS dimension for each cluster and
can be used as a preprocessing step, both for the inputs and the vectorial
outputs, for the construction of more accurate regressions and surrogate
modeling. The proposed approach is very versatile, especially the hierarchical
top-down clustering which can incorporate quite different criteria. The
methodology has been validated over a vast range of datasets, both scalar and
vector-valued, showing all the strengths and a possible weakness, in case of
radial symmetric functions. In all the test cases LAS achieved superior
performance with respect to the classical global approach.

Possible future lines of research can focus on the study of the extension of
this methods to nonlinear parameter space reduction techniques, or on the use of
more advanced clustering criteria.

\appendix
\setcounter{equation}{0}
\renewcommand\theequation{\arabic{equation}}
\section{Appendix}
\label{sec:appendix}
\subsection{Subspace Poincar\'e inequality}
\label{app:assumptions}
  The probabilistic Poincar\'e inequality for conditional probability
  densities or subspace Poincar\'e inequality~\cite{parente2020generalized} is
  valid at least for the following classes of absolutely continuous
  probability densities $\mupar$ with p.d.f. $\rho$.
  \begin{ass}
      \label{ass:pdf}
      The p.d.f $\rho:\mathcal{X}\subset\mathbb{R}^{n}\rightarrow\mathbb{R}$
      satisfies one of the following:
      \begin{enumerate}
          \item $\mathcal{X}$ is bounded connected open with Lipschitz
                boundary, $\rho$ is the uniform density distribution.
          \item $\mathcal{X}$ is convex and bounded, $\exists \,\delta,
                    D>0:\,0<\delta\leq\lVert\rho(\mathbf{x})\rVert_{L^{\infty}}\leq
                    D<\infty\ ,\forall \xx \in \mathcal{X}$,
          \item $\mathcal{X}=\mathbb{R}^{n}$, $\rho(\xx)\sim\exp(-V(\xx))$ where
                $V:\mathbb{R}^{n}\rightarrow
                (-\infty,\infty]\,,V\in\mathcal{C}^{2}$ is $\alpha$-uniformly
                convex,
                \begin{gather}
                    \mathbf{u}^{T}\text{Hess}(V(\xx))\mathbf{u}\geq\,\alpha
                    \lVert\mathbf{u}\rVert^{2}_{2},\quad\forall \xx,
                    \mathbf{u} \in \RR^n,
                \end{gather}
                where $\text{Hess}(V(\xx))$ is the Hessian of $V(\xx)$.
          \item $\mathcal{X}=\mathbb{R}^{n}$, $\rho(\xx)\sim\exp(-V(\xx))$ where
                $V$ is a convex function. In this case we require also $f$
                Lipschitz continuous.
      \end{enumerate}
  \end{ass}
  The last class of p.d.f. provides a weaker bound~(Lemma 4.3,
  \cite{parente2020generalized}) on the ridge approximation error. For the
  previous classes $i\in\{1, 2, 3, 4\}$ of p.d.f. an upper bound of the
  Poincar\'e constant $C_{P, i}$ is also provided:
  \begin{equation}
      C_{P, 1} = C_{P, 1}(\Omega),\qquad C_{P, 2}=\frac{D\text{diam}(\mathcal{X})}{\pi\delta}, \qquad C_{P, 3}=\frac{1}{\alpha},
  \end{equation}
  while the upper bound for $C_{P, 4}$ requires the definition of other
  quantities and is proved in Lemma 4.4~\cite{parente2020generalized}.

\subsection{Generalization of the upper bound on the approximation of the active subspace}
\label{app:upperbound approx AS}
  We want to make some brief considerations about the accuracy of the active
    subspace as eigen subspace of the correlation matrix approximated with Monte
    Carlo. If we use the notation $W_1\in\mathbb{R}^{n\times r},
    W_2\in\mathbb{R}^{n\times (n-r)}$ for the active and inactive subspaces
    (i.e. $P_r = W_1 W_{1}^{T},\, Id-P_r=W_2 W_{2}^{T}$) and
    $\hat{W_1}\in\mathbb{R}^{n\times r}, \hat{W_2}\in\mathbb{R}^{n\times (n-r)}$
    for the approximated active and inactive subspaces, we can bound the
    approximation error as done by Constantine in~\cite{constantine2015active}:
    assuming $f$ Lipschitz continuous, with high probability the following
    inequality is valid,
    \begin{equation}
        \label{eq:Constantine approx error}
        \text{dist}(\text{Im}(W_1), \text{Im}(\hat{W}_1))\lesssim
        \frac{4L\sqrt{n}(\log(n))^{\frac{1}{2}}}{N^{\frac{1}{2}}\lambda_{1}(\lambda_{r}-\lambda_{r+1})},
    \end{equation}
    where $L$ is the Lipschitz constant of $f$, $\{\lambda_1, \dots\lambda_n\}$
    are the non-negative eigenvalues of
    $\mathbb{E}_{\boldsymbol{\mu}_{i}}\left[\nabla f\otimes\nabla f\right]$
    ordered decreasingly, and $N$ is the number of Monte Carlo samples.

    The bound in \cref{eq:Constantine approx error} is obtained from Corollary
    3.8 and Corollary 3.10 in~\cite{constantine2015active}. It is founded on a
    matrix Bernstein inequality for a sequence of random uniformly bounded
    matrices (Theorem 6.1,~\cite{tropp2012user}) and on the Corollary 8.1.11
    from~\cite{golub2013matrix} that holds a bound on the sensitivity of
    perturbation of an invariant subspace.

    From Corollary 8.1.11 of~\cite{golub2013matrix}, a bound on the
    approximation error of the active subspace $W_1$ can be obtained
    making explicit
    $\lVert W^{T}_{2}EW_1\rVert_{F}$ with respect to the chosen numerical method
    for the discretization $\hat{C}$ of the integral
    $C=\mathbb{E}_{\boldsymbol{\mu}_{i}}\left[\nabla f\otimes\nabla f\right]$:
    in~\cite{constantine2015active} this has been done for the Monte Carlo
    method. In practice we could use quasi Monte Carlo sampling methods with
    Halton or Sobol' sequences~\cite{sullivan2015introduction}, since
    \begin{align*}
        \lVert W^{T}_{2}EW_1\rVert_{F} & \leq \sqrt{r(n-r)}\lVert W^{T}_{2}EW_1\rVert_{\text{max}} \\
        &\lesssim
        \sqrt{r(n-r)}D^{*}(\{x_i\}_i)\cdot\text{max}_{i,j\in\{1,\dots,n\}}(V^{\text{HK}}(\nabla
        f_i\nabla f_j))                                                                                     \\
                                       & \lesssim
        2\sqrt{r(n-r)}D^{*}(\{x_i\}_i)\cdot\text{max}(\vert f \vert)\cdot\text{max}_{i\in\{1,\dots,n\}}(V^{\text{HK}}(\nabla
        f_i))                                                                                               \\
                                       & \lesssim
        2\sqrt{r(n-r)}\cdot\text{max}_{i\in\{1,\dots,n\}}(V^{\text{HK}}(\nabla
        f_i))\frac{\log(N)^{n}}{N},
    \end{align*}
    where $V^{\text{HK}}$ is the Hardy--Krause variation and $D^{*}(\{x_i\}_i)$
    is the star discrepancy of the quasi random sequence $\{x_{i}\}_i$. For the
    above result we have imposed $\mathcal{X}=[0,1]^{n}$ but it can be extended
    to different domains~\cite{basu2016transformations}. Thus we obtain the
    bound
    \begin{align}
        \label{eq:quasi MC}
        \text{dist}(\text{Im}(W_1), \text{Im}(\hat{W}_1))&\lesssim
        \frac{4\lVert W^{T}_{2}EW_1\rVert_{F}}{\lambda_{r}-\lambda_{r+1}}\nonumber\\
        &\lesssim
        \frac{8L\sqrt{r(n-r)}\cdot\text{max}_{i\in\{1,\dots,n\}}(V^{\text{HK}}(\nabla
        f_i))}{\lambda_{r}-\lambda_{r+1}}\cdot\frac{\log(N)^{n}}{N}.
    \end{align}
    Other numerical integration rules can be chosen so that different regularity
    conditions on the objective function may appear on the upper bound of the
    error, as the Lipschitz constant on \cref{eq:Constantine approx error} or
    the Hardy--Krause variation on \cref{eq:quasi MC}. If the regularity of $f$
    is $\mathcal{C}^{s}$, we can also apply tensor product quadrature formulae
    or Smolyak's sparse quadrature rule~\cite{sullivan2015introduction}. For
    high-dimensional datasets and $f$ less regular, the estimate in
    \cref{eq:Constantine approx error} is the sharpest.


\subsection{Proof of \cref{cor:counterexample refinement}}
\label{sec:proof_counter}
\begin{proof} Let us use the notation
$h_{1}(x_{1}):=x_{1}(x_{1}+\epsilon)(x_{1}-\epsilon)$, and
$h_{2}(x_{2}):=\cos(\omega x_{2})$, it can be shown that
\begin{align*}
    \mathbb{E}_{\mu}\left[\nabla f\otimes\nabla f\right] &= \int_{B}
    \left(
        \begin{matrix}
        (h_{1}')^{2}(h_{2})^{2} & h_{1}h_{1}'h_{2}h_{2}'\\
        h_{1}h_{1}'h_{2}h_{2}' & (h_{1})^{2}(h_{2}')^{2}
    \end{matrix}
    \right) \, d\mu(\mathbf{x}) + \mu(A\cup C)\cdot\left(
            \begin{matrix}
                1 & 0\\
                0 & 0
        \end{matrix}\right)\\
        &=
        \left(
            \begin{matrix}
                \frac{2}{5}\epsilon^{5}\left(1+\frac{\sin(2\omega)}{2\omega}\right)& 0\\
            0 & \frac{4}{105}\omega^{2}\epsilon^{7}\left(1-\frac{\cos(2\omega)}{2\omega}\right)
        \end{matrix}
        \right)+ \mu(A\cup C)\cdot\left(
                \begin{matrix}
                    1 & 0\\
                    0 & 0
            \end{matrix}\right),
\end{align*}
thus, since we are considering a one dimensional active subspace, the active
eigenvector belongs to the set $\{(1, 0), (0, 1)\}$. Similarly we evaluate
\begin{align*}
    \mathbb{E}_{\mu_{B}}\left[\nabla f\vert_{B}\otimes\nabla f\vert_{B}\right] &=
        \left(
            \begin{matrix}
                \frac{8}{5}\epsilon^{4}\left(1+\frac{\sin(2\omega)}{2\omega}\right)& 0\\
            0 & \frac{16}{105}\omega^{2}\epsilon^{6}\left(1-\frac{\cos(2\omega)}{2\omega}\right)
        \end{matrix}
        \right),\\
    \mathbb{E}_{\mu_{A}}\left[\nabla f\vert_{A}\otimes\nabla
   f\vert_{A}\right] &= \mathbb{E}_{\mu_{C}}\left[\nabla
                               f\vert_{C}\otimes\nabla
                               f\vert_{C}\right] =
        \left(
            \begin{matrix}
                1 & 0\\
            0 & 0
        \end{matrix}
        \right),
\end{align*}
and conclude that there exist $\epsilon>0, \omega>0$ such that:
\begin{align}
    \frac{2}{5}
  \epsilon^{5}\left(1+\frac{\sin(2\omega)}{2\omega}\right) +
  4(1-\epsilon) \geq
  \frac{4}{105}\omega^{2}\epsilon^{7}\left(1-\frac{\cos(2\omega)}{2\omega}\right)\label{eq:condition1}
  ,\\
    \frac{8}{5}\epsilon^{4}\left(1+\frac{\sin(2\omega)}{2\omega}\right)\leq
  \frac{16}{105}\omega^{2}\epsilon^{6}\left(1-\frac{\cos(2\omega)}{2\omega}\right)\label{eq:condition2},
\end{align}
for example $\epsilon\sim 10^{-2}, \,\omega\sim 10^{4}$ (approximately
$10\epsilon^{-2}\leq\omega^2\leq10\epsilon^{-7}$). In this way, using the
notations of \cref{def:local ridge approximation}, we have
\begin{equation*}
P_{1, \mathcal{X}} = e_{1}\otimes e_{1},\quad P_{1, A}=P_{1, C}=e_{1}\otimes e_{1},\quad P_{1, B}=e_{2}\otimes e_{2},
\end{equation*}
and it follows that
\begin{equation*}
    \mathbb{E}_{\mu}\left[\|f-R_{AS}(r,
      f)\|^{2}\right]=\mathbb{E}_{\mu}\left[ f^{2}\vert_{B}\right]=(1/\mu(\mathcal{X}))\lVert
    h_{1}\rVert^{2}_{L^{2}(\mathcal{X}, \lambda)}\lVert
    h_{2}\rVert^{2}_{L^{2}(\mathcal{X}, \lambda)},
\end{equation*}
\begin{align*}
    \mathbb{E}_{\mu}\left[\|f-\mathbb{E}_{\mu}\left[f \vert P_{r}\right]\|^{2}\right]&=(1/\mu(\mathcal{X}))\lVert h_{1}\rVert^{2}_{L^{2}(\mathcal{X},
                      \lambda)}\lVert h_{2}-(1/\mu(\mathcal{X}))\int
                      h_{2}dx_{2}\rVert^{2}_{L^{2}(\mathcal{X},
                      \lambda)}\\
    &=(1/\mu(\mathcal{X}))\lVert h_{1}\rVert^{2}_{L^{2}(\mathcal{X},
      \lambda)}\left(\lVert h_{2}\rVert^{2}_{L^{2}(\mathcal{X},
      \lambda)}-\frac{7}{16}\left(\int h_{2}dx_{2}\right)^{2}\right),
\end{align*}
where $\lambda$ is the Lebesgue measure.
\end{proof}

\subsection{Computational complexity of hierarchical top-down}
\label{sec:complexity}

In \cref{tab:complexity} we report the computational complexity of the
hierarchical top-down clustering algorithm. We report the costs
divided by level of refinement.

\begin{table}[h]
    \caption{Computational complexity of hierarchical top-down
    clustering.\label{tab:complexity}}
    \centering
    \renewcommand{\arraystretch}{1.2}
    \begin{tabular}{ccc}
        \hline \hline
        Step      & Cost                 & Description \\
        \hline\hline
        Root      & $O(Nnp^2+Nn^2p+n^3)$ & AS          \\
                  & $O(N^3p^3)$          & GPR         \\
        \hline
        First refinement: & $O(N(N-k)^2)$                    & K-medoids \\
        $k$ from $m$ to $M$  & $O((N/k)np^2+(N/k)n^2p+n^3)$                   &
        AS \\
          & $O((N/k)^3p^3)$                    & GPR          \\
        \hline
        Intermediate refinements & - & - \\
        \hline
        Last refinement: & $O((N/k^{l-1})((N/k^{l-1})-k)^2)$ & K-medoids \\
        $k$ from $m$ to $M$  & $O((N/k^l)np^2+(N/k^l)n^2p+n^3)$ & AS \\
        for each one of the $m^{l-1}$ clusters  & $O((N/k^l)^3p^3)$ & GPR \\
        \hline \hline
    \end{tabular}
\end{table}

\section*{Acknowledgements}
This work was partially supported by an industrial Ph.D. grant sponsored by
Fincantieri S.p.A. (IRONTH Project), by MIUR (Italian ministry for university
and research) through FARE-X-AROMA-CFD project, and partially funded by European
Union Funding for Research and Innovation --- Horizon 2020 Program --- in the
framework of European Research Council Executive Agency: H2020 ERC CoG 2015
AROMA-CFD project 681447 ``Advanced Reduced Order Methods with Applications in
Computational Fluid Dynamics'' P.I. Professor Gianluigi Rozza.

\bibliographystyle{abbrvurl}


\end{document}